\documentclass[11pt]{article}
\usepackage[margin=0.8in]{geometry}

\usepackage{amsfonts, amsmath, amssymb, amsthm}
\usepackage{mathtools}

\usepackage{bm}
\usepackage{verbatim}
\usepackage{euscript}
\usepackage{graphicx}
\usepackage{paralist}

\usepackage{bbm}
\usepackage{nicefrac}

\usepackage{thmtools,thm-restate}

\DeclareMathOperator*{\argmax}{arg\,max}

\DeclarePairedDelimiter\norm{\lVert}{\rVert}

\usepackage{enumitem}
\usepackage{graphicx}
\usepackage{soul}
\usepackage{float}

\usepackage{xspace}
\usepackage{wrapfig}

\usepackage[usenames,dvipsnames]{color}
\usepackage{xcolor}

\usepackage[colorinlistoftodos]{todonotes}
\usepackage[colorlinks=true, allcolors=black]{hyperref}

\usepackage{algorithmicx, algpseudocode}
\usepackage[ruled]{algorithm}

\newtheorem{theorem}{Theorem}
\newtheorem{lemma}{Lemma}

\newtheorem{definition}{Definition}

\newtheorem{claim}[lemma]{Claim}

\hypersetup{
  colorlinks   = true, 
  urlcolor     = blue, 
  linkcolor    = blue, 
  citecolor   = blue 
}

\newcommand{\Argmax}[1]{\underset{#1}{\operatorname{arg}\,\operatorname{max}}\;}

\newcommand{\arm}{\mathcal{A}}
\newcommand{\Alg}{\textsc{Alg}}
\newcommand{\Hs}{\textsc{Hs}}

\newcommand{\AlgRS}{\textsc{Alg-RS}}
\newcommand{\AlgDRS}{\textsc{Alg-d-RS}}
\newcommand{\AlgDSK}{\textsc{d-LSky}}
\newcommand{\AlgDRSK}{\textsc{d-RSky}}

\newcommand{\vare}{\varepsilon}

\newcommand{\muhat}{\widehat{\mu}}

\newcommand{\game}{\textsc{Best}_{m,\vare,T}(\textsc{Alg})}
\newcommand{\gameone}{\textsc{Best}_{m,\vare,T_1}(\textsc{Alg})}

\title{\bfseries Optimal Algorithms for Range Searching over \\ Multi-Armed Bandits}
\author{Siddharth Barman\thanks{Indian Institute of Science. {\tt barman@iisc.ac.in}} \quad Ramakrishnan Krishnamurthy\thanks{Indian Institute of Science. {\tt kramakrishna@iisc.ac.in}} \quad Saladi Rahul\thanks{Indian Institute of Science. {\tt saladi@iisc.ac.in}}}

\date{}

\begin{document}

\maketitle 

\begin{abstract}
This paper studies a multi-armed bandit (MAB) version of the range-searching problem. In its basic form, range searching considers as input a set of points (on the real line) and a collection of (real) intervals. Here, with each specified point, we have an associated weight, and the problem objective is to find a maximum-weight point within every given interval. 

The current work addresses range searching with stochastic weights: each point corresponds to an arm (that admits sample access) and the point's weight is the (unknown) mean of the underlying distribution. In this MAB setup, we develop sample-efficient algorithms that find, with high probability, near-optimal arms within the given intervals, i.e., we obtain  PAC (probably approximately correct) guarantees. We also provide an algorithm for a generalization wherein the weight of each point is a multi-dimensional vector. The sample complexities of our algorithms depend, in particular, on the size of the {optimal hitting set} of the given intervals. 

Finally, we establish lower bounds proving that the obtained sample complexities are essentially tight. Our results highlight the significance of geometric constructs---specifically, hitting sets---in our MAB setting. 
\end{abstract}

\section{Introduction}
\label{section:introduction}
Range searching is a fundamental problem in computational geometry and database theory; see, e.g., \cite{ae99,a17,gjrs18,rt19}. This problem has been extensively studied over the past few decades, and applications of range searching (along with its variants) arise in numerous real-world domains, such as spatial databases~\cite{s90}, temporal 
databases~\cite{aae03}, and networking~\cite{lkgh03}.

In its basic form, range searching considers as input a set of points in $\mathbb{R}$---each with an associated weight---and a collection of (real) intervals.\footnote{The given intervals are not necessarily disjoint.} The problem's objective is to efficiently find, for every given interval $I$, an input point within $I$ with maximum weight. 

Note that this classic formulation of range searching addresses fixed (deterministic) weights, i.e., one assumes that the weight of each point is known a priori. By contrast, many modern applications require queries to be processed on uncertain (stochastic) data. As a stylized example of such a setting, consider a paid crowdsourcing platform wherein the competency of the participants (say, in performing a type of task) is not known beforehand. Here, each participant $a$ can be represented as a point $p_a \in \mathbb{R}$---which denotes the fixed monetary payment for $a$---and the competency (weight) of $a$ can be modeled as a stochastic quantity. In this setup, range searching corresponds to the natural problem of finding, for each queried price range (i.e., within each given interval), the most proficient crowdworker.

Motivated, in part, by such applications, we study range searching in the multi-armed bandit (MAB) framework. In particular, we address a natural formulation wherein each input point $p_a \in \mathbb{R}$ is associated with a random variable (drawn from an unknown distribution) and the weight of the point corresponds to the (a priori unknown) mean, $\mu_a$, of the random variable. That is, each given point is an arm and the objective is to identify an optimal (with respect to $\mu_a$) arm within each given interval. Here, as is standard in the MAB literature, we assume sample access to the random variables associated with the arms/points. The current work develops novel, {sample-efficient} algorithms for this stochastic version of range searching.

We also design algorithms for a pertinent generalization wherein the weight of each point is a $d$-dimensional vector, i.e., the random variables associated with every arm/point and, hence, the means $\mu_a$s are $d$-dimensional vectors. For the crowdsourcing example mentioned above, this generalization models settings in which the proficiency of each crowdworker needs to be assessed across $d$ different types of tasks.\footnote{As before, the monetary payment for each crowdworker is a fixed and known number.} In this multi-dimensional setting, the problem objective extends, from finding an optimal arm per interval, to computing a set of \emph{Pareto optimal} arms within each input interval. 

\paragraph{Our Contributions and Techniques.} Given the MAB nature of the problems, we obtain PAC-style (probably approximately correct) results, i.e., the developed algorithms compute, with high probability, near-optimal arms for every given interval.

The sample complexities of our algorithms depend, in particular, on the size of the optimal \emph{hitting set} of the given collection of intervals. Recall that a (finite) set of points, $\textsc{Hs} \subset \mathbb{R}$,\footnote{The hitting set is not required to be a subset of the input points.} is said to be a \emph{hitting set} for a collection of intervals $\mathcal{I}=\{I_1,I_2, \ldots, I_q\}$ iff $\textsc{Hs}$ contains a point from the interior of each interval $I_i \in \mathcal{I}$. Let $\tau$ be the size of a minimum-cardinality hitting set for the queried collection of intervals $\mathcal{I}$ and $n$ be the number of input points (arms). We show that\footnote{Note that, even though the single-dimensional-weights setting ($d=1$) is a special case of the multi-dimensional one, we present the settings separately to highlight the layered development of the algorithmic ideas.} 

\begin{itemize}
\item[(i)] For range searching with single-dimensional weights ($\mu_a \in \mathbb{R}$), i.e., for finding (near) optimal arms for every interval, the number of samples required by our algorithm is $O\left(\frac{n}{\varepsilon^2} \log \left( \frac{\tau}{\varepsilon \delta} \right) \right)$ (Theorem \ref{theorem:range-max}). Here, $\varepsilon>0$ is the approximation parameter and $\delta>0$ is the confidence parameter in the PAC guarantee.
\item[(ii)] For range searching with multi-dimensional weights ($\mu_a \in \mathbb{R}^d$), i.e., for finding (near) Pareto optimal arms for every interval, the number of samples required by our algorithm is $O\left(\tfrac{nd}{\varepsilon^2} \log(\tfrac{\tau d}{\varepsilon \delta}) \right)$ (Theorem \ref{theorem:range-pareto-optimal}). 
\item[(iii)] Finally, we establish lower bounds proving that the sample upper bounds obtained in our algorithms are essentially tight (Theorem \ref{thm:final-lb}). Specifically, these results show that a sample complexity dependence on $\log \tau$ is unavoidable, in general. 
\end{itemize}

Notably the above-mentioned upper bounds do not explicitly depend on the number of input intervals, $q$. Indeed, the size of an optimal hitting set, $\tau$, can be significantly smaller than $q$, e.g., $\tau = 1$ for a collection of intervals with a common intersection. Even in the worst case (with pairwise-disjoint intervals), we have $\tau = q$; one can always select the midpoint of each given interval to obtain a hitting set. Notably, a guarantee in terms of $\tau$ provides refined (and matching) upper and lower bounds. It also highlights a novel application of this geometric parameter in the current MAB setting. 

At a high level, our algorithm for single-dimensional weights proceeds by using the $\tau$ points in the optimal hitting set to construct $\tau+1$ pairwise-disjoint intervals (referred to as \emph{slabs}) that cover all the input intervals (which themselves might have multiple intersections).  A key idea then is to identify---through relevant subroutines---a candidate set of arms within each of the $\tau+1$ constructed slabs and judiciously combine them to find the desired set of near-optimal arms for the input intervals. The subroutines employed to populate the candidate set of arms are based on variations of a result of Cheu et al.~\cite{cheu2018skyline} that addresses the \emph{skyline problem} (with scalar weights). 

We build upon this design template, with additional technical insights, to obtain an algorithm for $d$-dimensional weights. Towards a subroutine, we develop a PAC algorithm for the $d$-dimensional version of the {skyline problem} (see Subsection~\ref{subsection:skyline-algos} for details). Note that the algorithm provided in \cite{cheu2018skyline} solely address the single-dimensional version of the skyline problem. Furthermore, even the scalar instantiation (obtained by setting $d=1$) of our skyline algorithm is distinct from that of Cheu et al.~\cite{cheu2018skyline}.

\paragraph{Related Work.} 
Given that the current paper is focused on finding (near) optimal arms in a sample-efficient manner, we focus on prior work with similar objectives. 
For a broad survey of MAB literature see, e.g., \cite{slivkins2019introduction}.   

Even-Dar et al.~\cite{even2006action} develop a PAC algorithm for the problem of identifying a best arm among all the $n$ given ones;
see also \cite{domingo2002adaptive}.
The  algorithm of Even-Dar et al. has a sample complexity of $O\left(\frac{n}{\vare^2} \log \frac{1}{\delta} \right)$ 
and a matching lower bound was obtained in~\cite{mannor2004sample}.
Multiple variants of best arm identification have also been studied in the literature; see, e.g.,~\cite{AudibertBR,bubeck2013multiple,kalyanakrishnan2010efficient,jamieson2014lil,garivier2016optimal,jun2016top,russo2016simple,yu2018pure,grover2018best}.
Note that best arm identification can be viewed as a special case of bandit range searching (by considering a single input interval that contains all the arms). 

Addressing settings in which each arm is associated with a $d$-dimensional random variable, the work of Auer et al.~\cite{auer2016pareto} provides a PAC algorithm---with sample complexity $O\left(\frac{n}{\vare^2} \log \frac{nd}{\delta} \right)$---for finding Pareto optimal arms. This setting is generalized in our range-searching framework with $d$-dimensional weights.  

Cheu et al.~\cite{cheu2018skyline} study a setup wherein the arms are positioned on the real line and they provide a PAC algorithm to find a \emph{skyline}. Bandit range searching generalizes this skyline problem as well. Complementarily, in the case of single-dimensional weights, one can use the algorithm of Cheu et al.~\cite{cheu2018skyline} as a subroutine for range searching. For multi-dimensional weights, we in fact develop a novel algorithm for computing skylines. This result generalizes the work of Cheu et al.~\cite{cheu2018skyline} (to multi-dimensional weights) and is potentially interesting in its own right. The technical connections between range searching and skylines are detailed in subsequent sections. 

Multiple recent results in computational geometry have addressed the (classic) range searching problem over uncertain data~\cite{txc07,acty09,akss18,lw16,cgzjcz17}. These works primarily address uncertainty in the point locations. By contrast, the current setting studies fixed point locations, but stochastic weights. \\

\noindent
\emph{Naive Algorithms:} A direct approach for the problem at hand is to first compute the empirical estimates of the weights (means) by sampling each arm $O\left( \frac{1}{\varepsilon^2} \log \frac{n}{\delta} \right)$ times and then report the optimal arms with respect to the estimates. This method requires $O\left( \frac{n}{\varepsilon^2} \log \frac{n}{\delta} \right)$ samples overall; see, e.g., the \textsc{Naive} algorithm in \cite{even2006action} for a PAC analysis. One can also show (by considering a union bound over the intervals, instead of the arms) that $O\left( \frac{n}{\varepsilon^2} \log \frac{q}{\delta} \right)$ samples suffice; here $q$ is the number of input intervals. Notably, the current work goes beyond these generic approaches and invokes problem-specific insights (e.g., the use of hitting sets) to develop refined upper bounds. Our tight lower bounds further substantiate the relevance of the obtained refinement. 

In the multi-dimensional weights case, the naive method would require $O\left( \frac{n}{\varepsilon^2} \log \frac{nd}{\delta} \right)$ samples. Comparing this sampling bound with Theorem \ref{theorem:range-pareto-optimal}, we can quantitatively identify settings wherein improvements over naive sampling are obtained. In particular, moderate values of $d$ highlight the relevant use cases of our result; notably, applications wherein $d \ll n$ are well-studied in computational geometry.

\section{Notation and Preliminaries}
\label{section:notation}
The paper studies settings in which one has access to independent draws from $n$ unknown distributions, each supported on $[0,1]$, i.e., we have sample access to $n$ arms. Throughout, $\mu_a \in [0,1]$ will be used to denote the (unknown) mean under the $a$th distribution and we will refer to $\mu_a$ as the \emph{weight} of arm $a \in [n]\coloneqq \{1,2, \ldots, n\}$. 

Furthermore, in our framework, each arm $a \in [n]$ is associated with a (fixed) point $p_a \in \mathbb{R}$. The set of points $\mathcal{P} \coloneqq \left\{p_1, \ldots, p_n\right\}$ is given to us as input, along with a collection of $q$ real intervals $\mathcal{I} \coloneqq \left\{ I_i = [\ell_i, r_i] \right\}_{i=1}^q$. Without loss of generality, we will assume that each given interval has a non-empty interior, i.e., for each interval $I_i = [\ell_i, r_i]$ we have $\ell_i < r_i$.
For any interval $J=[\ell, r]$, write $\arm(J)$ to denote the set of arms belonging to $J$, 
\begin{align*}
\arm(J) \coloneqq \{ a \in [n] \mid \ell  \leq p_a \leq r \}.
\end{align*}

The current work studies the \emph{bandit range searching} problem: given a set of points $\mathcal{P}=\{p_a\}_{a=1}^n$, a collection of intervals $\mathcal{I} = \{ I_i\}_{i=1}^q$, and sample-access to $n$ arms, find---for each given interval $I_i$---an arm $a \in \arm(I_i)$ with maximum weight $\mu_a$. 

Given the multi-armed bandit nature of this problem, we obtain PAC (probably approximately correct) guarantees; in particular, we develop sample-efficient algorithms that find, with high-probability, near-optimal arms for each given interval. We next define the PAC constructs, with approximation parameter $\varepsilon >0$ and confidence parameter $\delta >0$.

\begin{definition}[$\varepsilon$-optimality] 
For any interval $I$ and parameter $\varepsilon >0$, an arm $a \in \arm(I)$ is said to be $\varepsilon$-optimal for $I$ iff $\mu_a \geq \mu_b - \varepsilon$ for all $b \in \arm(I)$. 
\end{definition}

\begin{definition}[$(\varepsilon, \delta)$-PAC guarantee] An algorithm $\Alg$ is said to achieve the $(\varepsilon, \delta)$-PAC guarantee for the bandit range searching problem, iff---given any problem instance $(\mathcal{P}, \mathcal{I})$---$\Alg$ finds, with probability at least $(1- \delta)$, an $\varepsilon$-optimal arm for every interval $I \in \mathcal{I}$. 
\end{definition}

\paragraph{Range Searching with Multi-Dimensional Weights.} We will also develop algorithms for a generalization of bandit range searching wherein the weights are multi-dimensional vectors. In this generalization, for every arm, the underlying random variable is $d$-dimensional and component-wise supported on $[0,1]$. That is, for each arm $a \in [n]$ the (unknown) weight $\mu_a \in [0,1]^d$. The range aspects of the problem remain unchanged: each arm $a$ is associated with a (given) point $p_a \in \mathbb{R}$ and, as before, we are given a collection of intervals $\mathcal{I}=\left\{I_i = [\ell_i, r_i] \right\}_{i=1}^q$. In this setup, since the weights are $d$-dimensional vectors (instead of scalars), the problem objective extends to finding $\varepsilon$-Pareto optimal arms (instead of $\varepsilon$-optimal arms). 

Note that in the single-dimensional case, if $a$ was an $\varepsilon$-optimal arm in set $\arm$, then (hypothetically) adding $\varepsilon$ to the weight $\mu_a$, and keeping the weight of all the other arms unchanged, ensured proper optimality of $a$. The notion of $\varepsilon$-Pareto optimality (defined next) essentially builds upon this perspective. It deems a set of arms $T$ to be $\varepsilon$-Pareto optimal if (hypothetically) adding $\varepsilon \mathbbm{1}$ to the $d$-dimensional weight of each arm in $T$ (and keeping the weight of the other arms unchanged) gives us a legitimately Pareto optimal set. Formally, 

\begin{definition}[$\varepsilon$-Pareto optimality] 
\label{definition:eps-Pareto-optimal}
For any interval $I$ and parameter $\varepsilon >0$, a subset $T \subseteq \arm(I)$ is said to be $\varepsilon$-Pareto optimal for interval $I$ iff \\
\noindent
(a) for all arms $b \in \arm(I)$ there exists an arm $a \in T$ such that $\mu_a \geq \mu_b - \varepsilon \mathbbm{1}$, and \\
\noindent
(b) for all arms $a \in T$ there does not exist an arm $b \in \arm(I)$ such that $\mu_b \geq \mu_a + \varepsilon \mathbbm{1}$. 
\end{definition}
Here, the inequalities between weights are enforced component-wise and $\mathbbm{1}$ denotes the all-ones vector.  
We note that, up to pruning, a set $T$ that satisfies Definition~\ref{definition:eps-Pareto-optimal} also upholds the approximate-optimality criterion (specifically, Success Condition 2) studied in~\cite{auer2016pareto}.

Analogous to the single-dimensional case, we define the PAC guarantee under $d$-dimensional weights: an algorithm $\Alg$ is said to be $(\varepsilon, \delta)$-correct for bandit range searching with multi-dimensional weights, iff---given any problem instance $(\mathcal{P}, \mathcal{I})$---$\Alg$ finds, with probability at least $(1- \delta)$, an $\varepsilon$-Pareto optimal set of arms for every interval $I \in \mathcal{I}$. 

\paragraph{Hitting Set and Slabs.} 
A finite set of points $\Hs \subset \mathbb{R}$ is said to be a hitting set for a collection of intervals $\mathcal{I}=\left\{I_i = [\ell_i, r_i] \right\}_{i=1}^q$ iff for each $I_i = [\ell_i, r_i] \in \mathcal{I}$ there exists a point $e \in \Hs$ such that $\ell_i < e < r_i$.\footnote{Recall that for each input interval we have $\ell_i < r_i$.} 

Let $\Hs^*$ denote a minimum-cardinality hitting set for the given collection of intervals $\mathcal{I}$ and write $\tau \coloneqq |\Hs^*|$.\footnote{Given a collection of intervals $\mathcal{I}$, an optimal hitting set $\Hs^*$ (and, hence, $\tau$) can be computed in polynomial-time, via a greedy algorithm.} Note that the size of an optimal hitting set, $\tau$, is at most the number of intervals in $\mathcal{I}$, i.e., $\tau \leq q$. 

Our algorithms use the $\tau$ points in an optimal hitting set $\Hs^*$ and construct $\tau+1$ intervals that intersect only at their endpoints and partition the real line. In particular, with $e_1 < e_2 < \ldots < e_\tau$ denoting the set of points in $\Hs^*$, define the following collection of $\tau+1$ intervals: $S_0 \coloneqq (-\infty, e_1]$, $S_{\tau} \coloneqq [e_\tau, \infty)$, and $S_j \coloneqq [e_{j}, e_{j+1}]$ for $1 \leq j < \tau$.

We will refer to these $\tau+1$ intervals as \emph{slabs}. Note that, by construction, the slabs are indexed in increasing order of their left (or, equivalently, right) endpoints. Furthermore, the slabs have pairwise-disjoint interiors and their union covers all the input intervals, $\cup_{I \in \mathcal{I}} I \subseteq \cup_{j=0}^\tau S_j$. We will also utilize the following property of slabs, which follows from the fact that their endpoints constitute a hitting set: for every input interval $I \in \mathcal{I}$ there exists a sequence of \emph{at least two} slabs, $S_x, S_{x+1}, \ldots, S_y$, that intersect with $I$, i.e., 
\begin{align*}
I \cap S_t \neq \emptyset \qquad  \text{for $x \leq t \leq y$} \tag{P}
\end{align*}
Property (P) implies that every input interval gets partitioned among two or more slabs.

\section{Bandit Range Searching with Single-Dimensional Weights}
\label{section:single-dimension-algos}

Our algorithm, $\AlgRS$ (Algorithm \ref{algorithm:BRS}), begins by constructing slabs $\{S_j\}_j$ from an optimal hitting set of the input collection intervals $\mathcal{I}$. Then, $\AlgRS$ executes the following three subroutines for each slab $S_j$: \textsc{Best}($\cdot$), \textsc{LSkyline}($\cdot$), and \textsc{RSkyline}($\cdot$). The subset of arms collected through these executions serve as a candidate set $C$ and the algorithm computes the final set of $\varepsilon$-optimal arms by selecting the best (with respect to the empirically estimated weight) candidate arm within each input interval $I_i$. 

Given any interval $J$, approximation parameter $\varepsilon'>0$, and confidence parameter $\delta'>0$, the subroutine $\textsc{Best}(J, \varepsilon', \delta')$ finds, with probability $(1-\delta')$, an $\varepsilon'$-optimal arm for interval $J$ (i.e., the subroutine finds an $\varepsilon'$-optimal arm from the set of arms $\arm(J)$). $\AlgRS$ executes this subroutine over every slab $S_j=[e_j, e_{j+1}]$. We note that the algorithm of Even-Dar et al.~\cite{even2006action} provides a sample-efficient implementation of this subroutine.

The subroutines \textsc{LSkyline}($\cdot$) and \textsc{RSkyline}($\cdot$) solve the \emph{skyline problem}, which in fact can be viewed as a special case of bandit range searching. Specifically, in the skyline problem the points associated with the arms, $\{ p_a \in \mathbb{R}\}_a$, themselves generate the collection of query intervals $\mathcal{I}$: in the left-skyline problems we have $\mathcal{I} = \{ [p_a, p_{\infty}] \}_{a}$ and in the right-skyline problems $\mathcal{I} = \{ [p_{-\infty}, p_a]\}_{a}$, where $p_{\infty} \coloneqq \max_{a} \  p_a$ and $p_{-\infty} \coloneqq \min_{a} \ p_a$. The following definition encapsulates the $\varepsilon$-optimality requirement specifically for the skyline setting. 

\begin{definition}[$\varepsilon$-skyline]
\label{definition:skyline}
For an interval $J=[\ell,r]$ and parameter $\varepsilon >0$, a subset $L \subseteq \arm(J)$ is said to be an  $\varepsilon$-left-skyline within interval $J$ iff
(i) for each arm $b \in \arm(J)$, the set $L$ contains an $\varepsilon$-optimal arm for the interval $[p_b, r]$, and 
(ii) each arm in $L$ is $\varepsilon$-optimal within some interval $[p_b, r]$, with $b \in \arm(J)$. 
\end{definition}
Analogously, an $\varepsilon$-right-skyline, $R$, is  defined considering 
intervals $\{ [\ell, p_b]\}_{b \in \arm(J)}$, with $J=[\ell, r]$.\footnote{In Definition \ref{definition:skyline}, in place of the intervals $\{[p_b, r]\}_{b \in \arm(J)}$ one can equivalently use $\{[p_b, p_\infty]\}_{b \in \arm(J)}$, where $p_\infty = \max_{x \in \arm(J)} p_x$.}
Note that the intervals considered in Definition~\ref{definition:skyline} intersect and, hence, a single arm can be $\varepsilon$-optimal for multiple such intervals. In fact, one can verify that---irrespective of the number of arms in interval $J$---there always exists an  $\varepsilon$-left-skyline (and an $\varepsilon$-right-skyline) with cardinality $O(1/\varepsilon)$; recall that for every arm $a$ the weight (mean of the underlying random variables) $\mu_a \in [0,1]$.

For any interval $J$ along with parameters $\varepsilon'>0$ and $\delta'>0$, the subroutine \textsc{LSkyline}$(J, \varepsilon', \delta')$
(respectively \textsc{RSkyline}$(J, \varepsilon', \delta')$) finds, with probability $1-\delta'$, an $\varepsilon$-left-skyline (respectively $\varepsilon$-right-skyline) within $J$. To implement these subroutines, we can use the skyline algorithm of Cheu et al.~\cite{cheu2018skyline},\footnote{We note that the Cheu et al.~\cite{cheu2018skyline} do not formulate the skyline problem explicitly from a range-searching perspective. However, one can verify that Definition~\ref{definition:skyline} conforms to the construct studied in \cite{cheu2018skyline}.} or instantiate (with $d=1$) the one developed in Section~\ref{subsection:skyline-algos} for the $d$-dimension version of the skyline problem.     

The implementations of the three subroutines mentioned above---in addition to identifying the desired set of arms---provide an empirical estimate $\widehat{\mu}_a$ of the true weight (mean) $\mu_a$ of each returned arm $a$. Specifically, with the same success probability as before, for every returned arm $a$ we have an empirical estimate $\widehat{\mu}_a$ which satisfies $|\mu_a - \widehat{\mu}_a| \leq \varepsilon'$. 

\begin{algorithm}[h]
\small
\caption{$\AlgRS$: $(\varepsilon, \delta)$-PAC algorithm for Bandit Range Searching}
\label{algorithm:BRS} 
{\bf Input:} Set of points $\mathcal{P} = \{p_a \in \mathbb{R} \}_{a=1}^n$, collection of intervals $\mathcal{I}=\{ I_i = [\ell_i, r_i] \}_{i=1}^q$, and sample access to the $n$ arms,  along with parameters $\varepsilon >0$ and $\delta>0$. \\
{\bf Output:} Set of arms $\{\alpha_1, \alpha_2, \ldots, \alpha_q\}$
\begin{algorithmic}[1]
    \State Let $e_1 < e_2 < \ldots < e_\tau$ constitute a minimum-size hitting set for $\mathcal{I}$, and write $e_0 = -\infty$ and $e_{\tau+1} = +\infty$
    \State Define slabs $S_j = [e_j, e_{j+1}]$, for $0 \leq j  \leq \tau$ 
   \State For every slab $S_j$, with $0 \leq j \leq \tau$, set arm $\beta_j = \textsc{Best}\left(S_j, \frac{\varepsilon}{3}, \tfrac{\delta}{3(\tau+1)} \right)$ along with left skyline $L_j = \textsc{LSkyline}\left(S_j, \frac{\varepsilon}{3}, \tfrac{\delta}{3(\tau+1)}\right)$ and right skyline $R_j = \textsc{RSkyline}\left(S_j, \frac{\varepsilon}{3}, \tfrac{\delta}{3(\tau+1)}\right)$ 
         \State Set $C = \bigcup_{j=0}^\tau \left( \{ \beta_j \} \cup L_j \cup R_j \right)$ \label{line:candidates} \Comment{$C$ is the candidate set of arms populated across the slabs}
       \State For each interval $I_i \in \mathcal{I}$, select arm $\alpha_i = \Argmax{c \in C \cap \arm(I_i)} \ \widehat{\mu}_c$ \label{line:argmax} \Comment{$\widehat{\mu}_c$ is the empirical estimate of $\mu_c$}   
    \State \Return $\{\alpha_1, \alpha_2, \ldots, \alpha_q\}$
\end{algorithmic}
\end{algorithm}

\begin{restatable}{theorem}{TheoremSingleDim}
\label{theorem:range-max}
Given any problem instance $(\mathcal{P}, \mathcal{I})$ with $n$ arms, algorithm $\AlgRS$ draws $O\left(\tfrac{n}{\varepsilon^2} \log(\tfrac{\tau}{\varepsilon \delta}) \right)$ samples and achieves the $(\varepsilon, \delta)$-PAC guarantee for bandit range searching (with single-dimensional weights); here $\tau$ denotes the size of the minimum-cardinality hitting set for $\mathcal{I}$.
\end{restatable}

\begin{proof}
For any slab $S_j =[e_j, e_{j+1}]$, considered during the algorithm, write $n_j$ to denote the number of arms within it, $n_j \coloneqq |\arm(S_j)|$. The subroutine $\textsc{Best}\left(S_j, \frac{\varepsilon}{3}, \tfrac{\delta}{3(\tau+1)} \right)$ requires $O\left(\tfrac{n_j}{\varepsilon^2} \log\tfrac{3 (\tau+1)}{\delta}\right)$ samples and it finds, with probability $\left( 1 - \tfrac{\delta}{3(\tau+1)}\right)$, an arm $\beta_j \in \arm(S_j)$ and an estimate $\widehat{\mu}_{\beta_j} \in [0,1]$ with the properties that $\mu_{\beta_j} \geq \mu_a - \nicefrac{\varepsilon}{3}$, for all $a \in \arm(S_j)$, and $|\mu_{\beta_j} - \widehat{\mu}_{\beta_j}| \leq \nicefrac{\varepsilon}{3}$. 

The subroutines $\textsc{LSkyline}\left(S_j, \frac{\varepsilon}{3}, \tfrac{\delta}{3(\tau+1)}\right)$ and $\textsc{RSkyline}\left(S_j, \frac{\varepsilon}{3}, \tfrac{\delta}{3(\tau+1)}\right)$ require $O\left(\tfrac{n_j}{\varepsilon^2} \log\tfrac{3 (\tau+1)}{\varepsilon \delta}\right)$ samples each (Theorem 2, \cite{cheu2018skyline}). In particular, subroutine $\textsc{LSkyline} \left(S_j, \frac{\varepsilon}{3}, \tfrac{\delta}{3(\tau+1)}\right)$ finds, with probability $\left( 1 - \tfrac{\delta}{3(\tau+1)}\right)$, an $(\varepsilon/3)$-left-skyline $L_j$ along with estimates $\widehat{\mu}_c$, for every arm $c \in L_j$, such that $|\widehat{\mu}_c - \mu_c| \leq \nicefrac{\varepsilon}{3}$. A similar guarantee holds for $\textsc{RSkyline}\left(S_j, \frac{\varepsilon}{3}, \tfrac{\delta}{3(\tau+1)}\right)$, which finds $R_j$.

To upper bound the number of samples required by $\AlgRS$, we note that $\sum_{j=0}^{\tau+1} n_j \leq 2 n$. This inequality follows from the fact that any arm can be a part of at most two successive slabs. Now, summing over the sample complexities of the subroutines, across all the slabs, we get that overall $O\left(\tfrac{n}{\varepsilon^2} \log(\tfrac{\tau}{\varepsilon \delta}) \right)$ samples are drawn in the algorithm. 

Next, we complete the proof by showing that $\AlgRS$ achieves the $(\varepsilon, \delta)$-PAC guarantee. Write $\mathcal{E}$ to denote the event that the executions of the subroutines succeed across all the $\tau+1$ slabs. That is, under $\mathcal{E}$, for all the slabs, the subroutines find the desired set of arms along with accurate-enough estimates. Given that the success probability of each subroutine is  at least $\left( 1 - \tfrac{\delta}{3(\tau+1)}\right)$ and there are $3(\tau+1)$ subroutine instantiations, we get (via the union bound) that the probability $\mathbb{P} \{ \mathcal{E} \} \geq 1 - \delta$. 

We will show that, under event $\mathcal{E}$, for every input interval $I_i \in \mathcal{I}$, the arm $\alpha_i \in \arm(I_i)$ selected in Line~\ref{line:argmax} is $\varepsilon$-optimal for $I_i$.  Hence, the desired PAC guarantee holds. 

Consider any interval $I_i \in \mathcal{I}$ and let $a^* \in \arm(I_i)$ be an optimal arm for $I_i$, i.e., $\mu_{a^*} \geq \mu_a$ for all $a \in \arm(I_i)$. Property (P) of our slab construction (see Section~\ref{section:notation}) ensures that interval $I_i$ is partitioned among two or more slabs, $S_x, S_{x+1}, \ldots, S_y$; in particular, $x<y$. Here, $S_x$ is the left-most slab that intersects with $I_i$ and $S_y$ is the rightmost. The intermediate slabs (if any), $S_{x+1}, \ldots S_{y-1}$, are completely contained in $I_i$. We will perform a case analysis based on whether $a^*$ is contained in $S_x$, $S_y$, or one of the intermediate slabs. 

\noindent
Case {\rm I}: the arm $a^*$ is within an intermediate slab, i.e., $a^* \in \arm(S_t)$ with $x < t <y$. Here, under event $\mathcal{E}$, the arm $\beta_t$ (obtained by executing $\textsc{Best}$ on $S_t \subset I_i$) satisfies $\mu_{\beta_t} \geq \mu_{a^*} - \nicefrac{\varepsilon}{3}$. Note that $\beta_t \in \arm(I_i)$ is included in the candidate set $C$ (Line \ref{line:candidates}). Hence, $\beta_t$ is considered in the $\argmax$ executed in Line \ref{line:argmax} and we have $\widehat{\mu}_{\alpha_i} \geq \widehat{\mu}_{\beta_t}$. Also, event $\mathcal{E}$ gives us $|\mu_{\alpha_i} - \widehat{\mu}_{\alpha_i}| \leq \nicefrac{\varepsilon}{3}$ and $|\mu_{\beta_t} - \widehat{\mu}_{\beta_t}| \leq \nicefrac{\varepsilon}{3}$. Combining these inequalities we get $\mu_{\alpha_i} \geq \widehat{\mu}_{\alpha_i} - \nicefrac{\varepsilon}{3} \geq \widehat{\mu}_{\beta_t} - \nicefrac{\varepsilon}{3} \geq {\mu}_{\beta_t} - \nicefrac{2\varepsilon}{3} \geq \mu_{a^*} - \varepsilon$. Therefore, the selected arm $\alpha_i$ is $\varepsilon$-optimal for $I_i$.

\noindent
Case {\rm II}: $a^* \in \arm(S_x)$. In this case, we consider the $(\varepsilon/3)$-left-skyline, $L_x$, computed for slab $S_x$ (via subroutine $\textsc{LSkyline}$). Note that under $\mathcal{E}$, we have an arm $c \in L_x \subseteq \arm(S_x)$ which is $\varepsilon/3$-optimal for the interval $[p_{a^*}, r_x]$, where $r_x$ is the right endpoint of slab $S_x$; in particular, we have $\mu_c \geq \mu_{a^*} - \nicefrac{\varepsilon}{3}$. Since $r_x$ is strictly contained in the interval $I_i$, arm $c$ belongs to $\arm(I_i)$. Therefore, $c \in L_x \cap \arm(I_i) \subseteq C \cap \arm(I_i)$. This containment ensures that for the selected arm $\alpha_i$ we have $\widehat{\mu}_{\alpha_i} \geq \widehat{\mu}_{c}$. Also, event $\mathcal{E}$ gives us $|\mu_{\alpha_i} - \widehat{\mu}_{\alpha_i}| \leq \nicefrac{\varepsilon}{3}$ and $|\mu_{c} - \widehat{\mu}_{c}| \leq \nicefrac{\varepsilon}{3}$. Combining these inequalities we get $\mu_{\alpha_i} \geq \mu_{a^*} - \varepsilon$ and, hence, in this case as well $\alpha_i$ is $\varepsilon$-optimal for $I_i$. 

\noindent
Case {\rm III}: $a^* \in \arm(S_y)$. The analysis here relies on the $(\varepsilon/3)$-right-skyline, $R_y$, of slab $S_y$ and is otherwise identical to the previous case. For brevity, we omit repeating the details. 

Overall, we get that, under event $\mathcal{E}$, which holds with probability at least $(1-\delta)$, the arm $\alpha_i$ is $\varepsilon$-optimal for each interval $I_i \in \mathcal{I}$. This shows that $\AlgRS$ achieves the $(\varepsilon, \delta)$-PAC guarantee and completes the proof. 
\end{proof}

\section{Bandit Range Searching with Multi-Dimensional Weights}
\label{section:multi-dimension-algos}

This subsection develops our PAC algorithm for range-searching with $d$-dimensional weights. Specifically, for each arm $a \in [n]$,  the (unknown) weight/mean is a $d$-dimensional vector, $\mu_a \in [0,1]^d$, and the problem objective is to find $\varepsilon$-Pareto optimal arms (Definition \ref{definition:eps-Pareto-optimal}) within each input interval. 

The range aspects of the setup are as mentioned previously: each arm $a \in [n]$ is associated with a (given) point $p_a \in \mathbb{R}$ and we are also given a collection of intervals $\mathcal{I}=\left\{I_i = [\ell_i, r_i] \right\}_{i=1}^q$. 

Our algorithm, $\AlgDRS$ (Algorithm \ref{algorithm:dBRS} in Section \ref{subsection:pac-multi-dim}), begins by constructing slabs $\{S_j\}_j$ from an optimal hitting set of the input collection intervals $\mathcal{I}$. Then, $\AlgDRS$ executes the following two subroutines for each slab $S_j$: $\AlgDSK(\cdot)$ and $\AlgDRSK(\cdot)$. The subset of arms collected through these executions serve as a candidate set $C$ and the algorithm computes the desired set of arms by selecting Pareto optimal (with respect to the empirically estimated weights) \emph{candidate} arms within each input interval $I_i$. 

The subroutines $\AlgDSK$ and $\AlgDRSK$ solve the \emph{skyline problem}, which in fact can be viewed as a special case of bandit range searching. Specifically, in the skyline problem the points associated with the arms, $\{ p_a \in \mathbb{R}\}_a$, themselves generate the collection of query intervals $\mathcal{I}$: in the left-skyline problems we have $\mathcal{I} = \{ [p_a, p_{\infty}] \}_{a}$ and in the right-skyline problems $\mathcal{I} = \{ [p_{-\infty}, p_a]\}_{a}$, where $p_{\infty} \coloneqq \max_{a} \  p_a$ and $p_{-\infty} \coloneqq \min_{a} \ p_a$.

Prior work has only focused on the single-dimensional version of the skyline problem (see \cite{cheu2018skyline}). Hence, towards  implementations of the subroutines $\AlgDSK$ and $\AlgDRSK$, in Section \ref{subsection:skyline-algos} we develop a novel algorithm for computing $\varepsilon$-skylines with $d$-dimensional weights.

\subsection{Skyline with Multi-Dimensional Weights}
\label{subsection:skyline-algos}

In this section we develop a PAC algorithm for the skyline problem with weights in $[0,1]^d$. 

\begin{definition}[$\varepsilon$-skyline for $d$-dimensional weights]
\label{definition:skyline-multi-d}
For an interval $J=[\ell,r]$ and parameter $\varepsilon >0$, a set of arms $L \subseteq \arm(J)$ is said to be an $\varepsilon$-left-skyline within interval $J$ iff 
(i) for each arm $b \in \arm(J)$, the set $L$ contains an $\varepsilon$-Pareto optimal set $T$ for the interval $[p_b, r]$, and 
(ii) each arm in $L$ is $\varepsilon$-Pareto optimal for some interval $[p_b, r]$, with $b \in \arm(J)$.   
\end{definition}

Analogously, an $\varepsilon$-right-skyline, $R$, is defined considering
 intervals $\{ [\ell, p_b]\}_{b \in \arm(J)}$. The following characterization of $\varepsilon$-skylines will be used in the analysis; a proof of this proposition appears in Appendix~\ref{appendix:skyline-equivalence-proof}.  

\begin{restatable}{proposition}{PropositionSkylineCharacter}
\label{propsition:syline-multi-d-alternative}
A set of arms $L \subseteq \arm(J)$ is an $\varepsilon$-left-skyline within interval $J=[\ell, r]$ iff 
(i) for every arm $b \in \arm(J)$, there exists an arm $\beta \in L$ that satisfies $p_\beta \geq p_b$ and $\mu_\beta \geq  \mu_b - \varepsilon \mathbbm{1} $, and (ii) for each arm $\beta \in L$, there does not exist an arm $x \in \arm(J)$ such that $p_x \geq p_\beta$ and $\mu_x \geq \mu_\beta + \varepsilon \mathbbm{1}$.
\end{restatable}

In Proposition \ref{propsition:syline-multi-d-alternative}---and in the remainder of this section---the inequalities are enforced component-wise. At a high level, the first condition of the proposition requires that for each arm $b$ in the given interval, $L$ contains a dominating (componentwise and within a factor of $\varepsilon$) point to the right of $b$. The second condition mandates that any arm in $L$ is itself not dominated (componentwise and beyond a factor of $\varepsilon$) by arms to its right.

Our algorithm, $\AlgDSK$ (Algorithm \ref{Algorithm:SkylineUsingEpsCubes}), starts by setting $A_1$ as the set of all arms within the given interval and $\varepsilon_1$ as an initial approximation parameter. The algorithm then considers a partition of $[0,1]^d$ into hypercubes of side length $\varepsilon_1/4$. Drawing a conservative number of samples, the algorithm finds an estimate $\widehat{\mu}_a \in [0,1]^d$ for each arm $a \in A_1$ and partitions the set of arms $A_1$ itself based on how the estimates fall across different hypercubes. 

For any hypercube $B$, let $B(A_1)$ be the set of arms (in $A_1$) whose estimates are within $B$. We will establish that, if the cardinality of $B(A_1)$ is sufficiently large, then we can drop half of the arms in $B(A_1)$ from consideration, and still not loose $\varepsilon$-Pareto dominating arms. In particular, from every large-sized set $B(A_1)$, the algorithm removes from consideration half of the arms $x \in B(A_1)$ whose $p_x$ value is less than the median of $\{ p_y \}_{y \in B(A_1)}$. This update ensures that the number of arms under consideration decreases geometrically, and at the same time, the set of arms that remain, say $A_2$, continues to include an approximate skyline. The algorithm repeats analogous steps over $A_2$, with an updated approximation parameter $\varepsilon_2 = 3\varepsilon_1/4$. Iteratively, the algorithm continues until the number of remaining arms is sufficiently small. At this point, the algorithm concludes by sampling the final set of arms $A_T$ (to obtain accurate-enough estimates of their weights) and returning a skyline (based on the final estimates) within $A_T$. 

We will now define constructs that will be useful in the design and analysis of our algorithm. For $\eta \in (0,1]$, define a lattice $L(\eta)$ as the set of vectors in $[0,1]^d$ whose components are integer multiples of $\eta$, i.e., $L(\eta) \coloneqq \left\{  (\kappa_1 \eta, \kappa_2 \eta, \ldots, \kappa_d \eta) \in [0, 1]^d \mid \kappa_1, \ldots, \kappa_d \in \mathbb{Z} \right\}$. For each vector $g \in L(\eta)$, define a (hyper) cube around it as 
$B(\eta, g) \coloneqq \left\{\nu \in [0,1]^d \mid \norm{\nu - g}_{\infty} \leq \tfrac{\eta}{2} \right\}$. Note that the collection of cubes $\{B(\eta, g)\}_{g \in L(\eta)}$ forms a cover of $[0,1]^d$. Faces of these cubes can intersect. However, for ease of presentation and given that the interiors of the cubes are pairwise disjoint, we will consider them to constitute a partition of $[0,1]^d$; in particular, by breaking ties arbitrarily, we will assume that  each $\nu \in [0,1]^d$ belongs to exactly one cube in $\{B(\eta, g)\}_{g \in L(\eta)}$.

In each iteration $t$, with approximation parameter $\varepsilon_t >0$, the algorithm considers the collection of cubes obtained by setting $\eta = \varepsilon_t/4$, i.e., it considers 
\begin{align*}
\mathcal{B}_t \coloneqq \left\{B \left( \frac{\varepsilon_t}{4}, g \right) \right\}_{g \in L\left( \frac{\varepsilon_t}{4}\right)}.
\end{align*}
Note that the cardinality of $\mathcal{B}_t$ is at most  $\left( \frac{4}{\varepsilon_t} \right)^d$.
\footnote{The number of hypercubes is actually $(\lceil \tfrac{4}{\varepsilon_i} \rceil)^d$, but we assume $\tfrac{4}{\varepsilon_i}$ is an integer for ease of analysis.} 

\begin{algorithm}[h]
\small
\caption{$\AlgDSK$: finds $\varepsilon$-left-skyline with probability at least $1-\delta$}
\label{Algorithm:SkylineUsingEpsCubes}
{\bf Input:} Interval $J$, with sample access to arms in $\arm(J)$, points $\{p_a\}_{a \in \arm(J)}$, and parameters $\varepsilon, \delta > 0$. {\bf Output:} Set of arms $S \subseteq \arm(J)$ 
\begin{algorithmic}[1]
    \State Initialize set $A_1 = \arm(J)$, parameters $\varepsilon_{1} = \nicefrac{\varepsilon}{5}$, $\delta_{1} = \nicefrac{\delta}{2}$, and iteration count $t =1$
    \While{$|A_t| > 120\left(\tfrac{4}{\varepsilon_t}\right)^d$} \label{line:while-loop-condition}
        \State For each arm $a \in A_t$, find an empirical estimate, $\widehat{\mu}_a \in [0,1]^d$, of its weight, by sampling the arm $\dfrac{8}{\varepsilon_t^2} \log \left( \left(\dfrac{4}{\varepsilon_t}\right)^d \dfrac{50d}{\delta_t} \right)$ times \label{line:pull}
        \State Initialize the set of arms to drop, $D_t = \emptyset$
\ForAll{cubes $B \in \mathcal{B}_t$}
\State Write $B(A_t) \coloneqq \left\{ x \in A_t \mid \muhat_x \in B \right\} $    \Comment{$B(A_t)$ is defined using the \emph{estimated} weights} \label{line:bat}
\State  \label{line:dropping-threshold}  If $|B(A_t)| > \frac{|A_t|}{10} \left(\frac{\varepsilon_t}{4}\right)^d$, then $D_t \leftarrow D_t \cup \left\{ x \in B(A_t) : p_x \text{ is less than the median of } \{p_y\}_{y \in B(A_t)} \right\}$
\EndFor
        \State Set $A_{t+1} = A_t \setminus D_t$, $\varepsilon_{t+1} = \nicefrac{3\varepsilon_{t}}{4}$, $\delta_{t+1} = \nicefrac{\delta_t}{2}$, and update $t \longleftarrow t + 1$ \label{line:elimination}
    \EndWhile
  \State For each arm $a \in A_t$, find an estimate, $\widehat{\mu}_a \in [0,1]^d$, of its weight, by sampling the arm $\dfrac{1}{\varepsilon_t^2} \log \left(\dfrac{|A_t| d}{\delta_t} \right)$ times \label{line:pruning}
    \State  Set $\widetilde{D} \coloneqq \{ z \in A_t  \mid p_z \leq p_x \text{  and } \widehat{\mu}_z < \widehat{\mu}_x \text{ for some } x \in A_t \}$ \Comment{With respect to the estimates, condition (ii) of Proposition \ref{propsition:syline-multi-d-alternative} does \emph{not} hold for arms in $\widetilde{D}$}
    \State \Return $A_t \setminus \widetilde{D}$
\end{algorithmic}
\end{algorithm}

In the analysis, we will consider a \emph{representative arm} $r(B)$ for each cube $B \in \mathcal{B}_t$: among all arms $b$ whose weight $\mu_b$ is in $B$, arm $r(B)$ is the one with the maximum $p_{r(B)}$ value. Specifically,   
\begin{definition}[Representative Arms] In any iteration $t$---with set of arms $A_t$ under consideration---an arm $r(B) \in A_t$ is said to be the representative of cube $B \in \mathcal{B}_t$ iff $\mu_{r(B)} \in B$ and $p_{r(B)} \geq p_b$ for all arms $b \in A_t$ whose weight $\mu_b \in B$. 
\end{definition}
Write $R_t \coloneqq \{r(B)\}_{B \in \mathcal{B}_t}$ to denote the set of all representatives in iteration $t$. Note that representative arms are defined considering the exact (but unknown) weights $\mu_b \in [0,1]^d$. Still the definition is relevant, since we use these arms solely for the purposes of analysis, and not in the algorithm (which only has access to estimated weights $\widehat{\mu}_b$s). Also, note that the size of $R_t$ is at most that of $\mathcal{B}_t$ and, hence, $|R_t|  \leq \left( \frac{4}{\varepsilon_t} \right)^d$.

The next lemma (Lemma \ref{lemma:skyline-initial-representative}) highlights a significance of representative arms--it asserts that $R_1$ (the set of representative arms in the first iteration) satisfies condition (i) of Proposition \ref{propsition:syline-multi-d-alternative}, with an (absolute) approximation factor of $\nicefrac{\varepsilon}{20}$.
Then, Lemma \ref{lemma:iterative-pac-guarantee} (stated below) guarantees that, with high probability, the algorithm essentially maintains condition (i)  between consecutive $R_t$s; Appendix \ref{appendix:reps} contains a proof of this lemma.

\begin{restatable}{lemma}{LemmaInitialReps}
\label{lemma:skyline-initial-representative}
For every arm $b \in \arm(J)$, there exists an arm $r \in R_1$ that satisfies $p_r \geq p_b$ and $\mu_r \geq  \mu_b - \frac{\varepsilon}{20} \mathbbm{1}$.
\end{restatable}
\begin{proof}
For any arm $b \in \arm(J) = A_1$, write $B$ to denote the cube that contains $\mu_b \in [0,1]^d$. This cube's representative arm $r(B)$ satisfies $p_{r(B)} \geq p_b$ and $\mu_{r(B)}  \geq \mu_b - \tfrac{\varepsilon_1}{4} \mathbbm{1}$; recall that the side length of $B$ is $\nicefrac{\varepsilon_1}{4}$. The initialization $\varepsilon_1 = \nicefrac{\varepsilon}{5}$ gives us $\mu_{r(B)}  \geq \mu_b - \tfrac{\varepsilon}{20} \mathbbm{1}$. 

That is, for each arm $b \in \arm(J)$ there exists an arm $r(B) \in R_1$ that satisfies condition (i) of Proposition \ref{propsition:syline-multi-d-alternative} and the lemma follows.\footnote{Indeed, the algorithm does not have a priori access to $\mu_r$s and, hence, cannot explicitly identify $R_1$ within $A_1$.} 
\end{proof}

\begin{restatable}{lemma}{LemmaRepsToReps}
\label{lemma:iterative-pac-guarantee}
In any iteration $t$, with probability $1-\delta_t$, for every representative arm $r \in R_t$, there exists an arm $\gamma \in R_{t+1}$ such that $\mu_{\gamma}  \geq \mu_r - \varepsilon_t \mathbbm{1}$ and $p_{\gamma} \geq p_r$.
\end{restatable}

The following lemma bounds the number of samples required by $\AlgDSK$. Its proof is deferred to Appendix \ref{appendix:skyline-sample-complexity-proof}.
\begin{restatable}{lemma}{LemmaSkylineSampleComplexity}
\label{lemma:skyline-sample-complexity}
For any given interval $J$ with $m$ arms ($m = |\arm(J)|$), $\AlgDSK$ draws $O\left(\frac{md}{\varepsilon^2}\log \frac{d}{\varepsilon \delta}\right)$ samples.
\end{restatable}

We next state and establish the main result of this section.

\begin{restatable}{theorem}{TheoremSkyline}
\label{theorem:skyline-alg-correctness}
Given any interval $J$ with $m$ arms, $\AlgDSK$ draws $O\left(\tfrac{md}{\varepsilon^2} \log \tfrac{d}{\varepsilon \delta} \right)$ samples and returns an $\varepsilon$-left-skyline within $J$ with probability at least $1-\delta$.
\end{restatable} 
\begin{proof}
The stated sample complexity of $\AlgDSK$ follows directly from Lemma \ref{lemma:skyline-sample-complexity}. Hence, we complete the proof of the theorem by showing that $\AlgDSK$ achieves the desired PAC guarantee.    

Write $T$ to denote the total number of iterations of the while loop in $\AlgDSK$. Hence, $A_T$ is the set of arms with which $\AlgDSK$ exits the while loop and $R_T$ is the set of representatives within $A_T$. Furthermore, $\AlgDSK$ returns $A_T\setminus \widetilde{D}$ as an $\varepsilon$-left-skyline; here, $\widetilde{D} \coloneqq \{ z \in A_T \mid p_z \leq p_x \text{  and } \widehat{\mu}_z < \widehat{\mu}_x \text{ for some } x \in A_T \}$. 

We will establish that (with the stated probability) $A_T\setminus \widetilde{D}$ satisfies both the conditions in Proposition \ref{propsition:syline-multi-d-alternative}, i.e., it is an $\varepsilon$-left-skyline. 

Towards this, first note that set $R_T \subseteq A_T$ satisfies condition (i) of Proposition \ref{propsition:syline-multi-d-alternative}--this follows from Lemma \ref{lemma:skyline-initial-representative} and Lemma \ref{lemma:iterative-pac-guarantee}. Indeed, Lemma \ref{lemma:skyline-initial-representative} implies that $R_1$ satisfies Proposition \ref{propsition:syline-multi-d-alternative}; in particular, we have that for each arm $b \in \arm(J)$, there exists an arm $r \in R_1$ such that $p_r \geq p_b$ and $ \mu_r \geq  \mu_b - \frac{\varepsilon}{20} \mathbbm{1}$. Lemma \ref{lemma:iterative-pac-guarantee} holds for each iteration $t$ with probability at least $1- \delta_t$. Hence, via the union bound, we get that the guarantee in Lemma \ref{lemma:iterative-pac-guarantee} holds for all the iterations with probability at least $1 - \sum_{t=1}^{T-1} \delta_t$. Chaining these guarantees, one obtains: for each arm $r \in R_1$, there exists an arm $\gamma \in R_T$ such that $p_\gamma \geq p_r$ and $ \mu_\gamma \geq  \mu_r - \sum_{t=1}^{T-1} \varepsilon_t \mathbbm{1}$. This observation, along with the above-mentioned property of $R_1$, shows that (with probability at least $1 - \sum_{t=1}^{T-1} \delta_t$) $R_T$ satisfies condition (i) of Proposition \ref{propsition:syline-multi-d-alternative}: for each arm $b \in \arm(J)$, there exists an arm $\gamma \in R_T$ for which $p_\gamma \geq p_b$ and $ \mu_\gamma \geq  \mu_b - \frac{\varepsilon}{20} \mathbbm{1} - \sum_{t=1}^{T-1} \varepsilon_t \mathbbm{1}$. By definition, $R_T \subseteq A_T$ and, hence, condition (i) of Proposition \ref{propsition:syline-multi-d-alternative} holds for $A_T$ (with probability at least $1 - \sum_{t=1}^{T-1} \delta_t$).

Recall that  in Line \ref{line:pruning} each arm in $A_T$ is sampled $\dfrac{1}{\varepsilon_T^2} \log \left(\dfrac{|A_T| d}{\delta_T} \right)$ times. Hence, with probability $1 - \delta_T$, we have $\|\widehat{\mu}_a - \mu_a \|_\infty \leq \varepsilon_T$ for every arm $ a \in A_T$.\footnote{Here, $\widehat{\mu}_a \in [0,1]^d$  is the estimate computed in Line \ref{line:pruning}.} For the rest of the proof we will assume that these bounds hold and $A_T$ satisfies condition (i). It suffices to prove the stated claim under this assumption, since it holds with probability at least $ 1 - \sum_{t=1}^{T-1} \delta_t - \delta_T \geq 1 - \delta$.

In particular, we will complete the proof by showing that $A_T\setminus \widetilde{D}$ continues to satisfy condition (i)  and this returned set also bears condition (ii) of Proposition \ref{propsition:syline-multi-d-alternative}. For condition (i), note that for any arm $b \in \arm(J)$ there exists an arm $\gamma \in A_T$ with the property that $p_\gamma \geq p_b$ and $\mu_\gamma \geq \mu_b - \frac{\varepsilon}{20} \mathbbm{1} - \sum_{t=1}^{T-1} \varepsilon_t \mathbbm{1}$. If $\gamma \notin \widetilde{D}$, then the condition holds for arm $\gamma$ in $A_T\setminus \widetilde{D}$ as well. Otherwise, consider an arm ${\sigma} \in A_T\setminus \widetilde{D}$ for which we have $p_{\sigma} \geq p_\gamma$ and $\widehat{\mu}_{\sigma} > \widehat{\mu}_\gamma$--such an arm is guaranteed to exist. Using the fact that, for all arms in $A_T$, the estimates are at most $\varepsilon_T$ away from the (exact) weights, we obtain 
\begin{align*}
\mu_\sigma &\geq \mu_\gamma - 2 \varepsilon_T \mathbbm{1} \geq \left(\mu_b - \frac{\varepsilon}{20} \mathbbm{1} - \sum_{t=1}^{T-1} \varepsilon_t \mathbbm{1}\right) - 2 \varepsilon_T \mathbbm{1} \\
& \geq \mu_b - \frac{\varepsilon}{20} \mathbbm{1} - \sum_{t=1}^{\infty} \varepsilon_t \mathbbm{1} \geq \mu_b - \frac{\varepsilon}{20} \mathbbm{1} - \frac{4\varepsilon}{5} \mathbbm{1} \geq \mu_b -  \varepsilon \mathbbm{1} \\
\end{align*}
Also, note that $p_\sigma \geq p_\gamma \geq p_b$ and, hence, condition (i) holds for $A_T\setminus \widetilde{D}$ (with an approximation factor of $\varepsilon$). 

Finally, we prove that $A_T\setminus \widetilde{D}$ bears condition (ii) of Proposition \ref{propsition:syline-multi-d-alternative}. Assume, towards a contradiction, that condition (ii) is violated for an arm $\beta \in A_T\setminus \widetilde{D}$, i.e., there exists arm $x \in \arm(J)$ such that $p_x \geq p_\beta$ and $\mu_x \geq \mu_\beta + \varepsilon \mathbbm{1}$. Now, since condition (i) holds for $A_T$, there exists arm $\gamma \in A_T$ that satisfies $p_\gamma \geq p_x$ and $\mu_\gamma \geq \mu_x - \frac{\varepsilon}{20} \mathbbm{1} - \sum_{t=1}^{T-1} \varepsilon_t \mathbbm{1}$. Using the fact that for arms $\beta, \gamma \in A_T$, the estimates are at most $\varepsilon_T$ away from the (exact) weights, we obtain $\widehat{\mu}_\gamma > \widehat{\mu}_\beta$. However, this strict inequality and $p_\gamma \geq p_x \geq p_\beta$ would imply that $\beta \in \widetilde{D}$, i.e., contradict the containment of $\beta$ in $A_T\setminus \widetilde{D}$. Therefore, $A_T\setminus \widetilde{D}$ must satisfy condition (ii). 

Overall, we get that the returned set $A_T\setminus \widetilde{D}$ bears both the conditions in Proposition \ref{propsition:syline-multi-d-alternative} and, hence, it is an $\varepsilon$-left-skyline. This completes the proof. 
\end{proof}

Analogous to Theorem \ref{theorem:skyline-alg-correctness}, one can achieve a PAC guarantee for $\varepsilon$-right-skylines by executing $\AlgDSK$ on negated $p_a$ values.

\subsection{Algorithm for Multi-Dimensional Weights}
\label{subsection:pac-multi-dim}

This subsection develops an algorithm for bandit range searching with $d$-dimensional weights. As mentioned previously, this result is obtained by reducing range searching to finding $\varepsilon$-skylines over slabs.

In particular, given interval $J$ and parameters $\varepsilon', \delta'>0$, the subroutine $\AlgDSK(J, \epsilon', \delta')$ finds, with probability at least $1-\delta'$, an $\varepsilon'$-left-skyline within $J$. Similarly, $\AlgDRSK(J, \epsilon', \delta')$ finds, with probability at least $1-\delta'$, an $\varepsilon'$-right-skyline within $J$. These  subroutines---in addition to identifying the desired set of arms---provide an empirical estimate $\widehat{\mu}_a \in [0, 1]^d$ of the true weight (mean) $\mu_a \in [0,1]^d$ of each returned arm $a$. Specifically, subsumed within the same success probability as before, for every returned arm $a$ we have an empirical estimate $\widehat{\mu}_a$ which satisfies $\| \mu_a - \widehat{\mu}_a \|_\infty \leq \varepsilon'$. This property is indeed satisfied by the algorithms in Section \ref{subsection:skyline-algos}. 

By executing these two subroutines over each slab, $\AlgDRS$ populates a set of candidate arms $C$ along with their weight estimates. Finally, for each input interval $I_i$, the algorithm considers the set of candidate arms within $I_i$ and returns the ones that form a Pareto optimal set with respect to the estimates. 

\begin{algorithm}[h]
\small
\caption{$\AlgDRS$: $(\varepsilon, \delta)$-PAC algorithm for range searching $d$-dimensional weights}
\label{algorithm:dBRS}
{\bf Input:} Set of points $\mathcal{P} = \{p_a \in \mathbb{R} \}_{a=1}^n$, collection of intervals $\mathcal{I}=\{ I_i = [\ell_i, r_i] \}_{i=1}^q$, and sample access to the $n$ arms,  along with parameters $\varepsilon >0$ and $\delta>0$. \\
{\bf Output:} Subsets of arms $\{T_i \subseteq [n] \}_{i=1}^q$
\begin{algorithmic}[1]
	\State Let $e_1 < e_2 < \ldots < e_\tau$ constitute a minimum-size hitting set for $\mathcal{I}$, and write $e_0 = -\infty$ and $e_{\tau+1} = +\infty$
	\State Define slabs $S_j \coloneqq [e_j, e_{j+1}]$, for $0 \leq j  \leq \tau$ 
	\State For every slab $S_j$, with $0 \leq j \leq \tau$, set left skyline $L_j = \AlgDSK\left(S_j, \frac{\varepsilon}{3}, \tfrac{\delta}{2(\tau+1)}\right)$ and right skyline $R_j = \AlgDRSK\left(S_j, \frac{\varepsilon}{3}, \tfrac{\delta}{2(\tau+1)}\right)$ 
       
	\State Set of candidate arms $C = \bigcup_{j=0}^\tau \left( L_j \cup R_j \right)$ \label{line:d-candidates} 
		\State \label{line:Pareto-prune} For each interval $I_i \in \mathcal{I}$, set $D_i = \left\{ x \in C \cap \arm(I_i) \mid \widehat{\mu}_x < \widehat{\mu}_y \text { for some } y \in C \cap \arm(I_i) \right\}$ and $T_i  = \left(C \cap \arm(I_i) \right) \setminus D_i$ \Comment{With respect to the estimates, $D_i$ is the set of Pareto dominated arms}
 	\State \Return $ \{T_1,\ldots, T_q\}$
\end{algorithmic}
\end{algorithm}

The following theorem asserts that $\AlgDRS$ achieves the desired PAC guarantee in a sample-efficient manner.

\begin{restatable}{theorem}{TheoremMultiDim}
\label{theorem:range-pareto-optimal}
Given a problem instance $(\mathcal{P}, \mathcal{I})$ with $n$ arms,  algorithm $\AlgDRS$ draws $O\left(\tfrac{nd}{\varepsilon^2} \log(\tfrac{\tau d}{\varepsilon \delta}) \right)$ samples and achieves the $(\varepsilon, \delta)$-PAC guarantee for bandit range searching with $d$-dimensional weights; here $\tau$ denotes the size of the minimum cardinality hitting set for $\mathcal{I}$.
\end{restatable}
\begin{proof}
For any slab $S_j =[e_j, e_{j+1}]$, considered during the algorithm, write $n_j$ to denote the number of arms within it, $n_j \coloneqq |\arm(S_j)|$. The subroutines $\AlgDSK\left(S_j, \frac{\varepsilon}{3}, \tfrac{\delta}{2(\tau+1)}\right)$ and $\AlgDRSK\left(S_j, \frac{\varepsilon}{3}, \tfrac{\delta}{2(\tau+1)}\right)$ require $O\left(\tfrac{n_j \ d}{\varepsilon^2} \log\tfrac{ 2d (\tau+1)}{\varepsilon \delta}\right)$ samples each. In particular, subroutine $\AlgDSK\left(S_j, \frac{\varepsilon}{3}, \tfrac{\delta}{2(\tau+1)}\right)$ finds, with probability $\left( 1 - \tfrac{\delta}{2(\tau+1)}\right)$, an $(\varepsilon/3)$-left-skyline $L_j$ along with estimates $\widehat{\mu}_c$, for every arm $c \in L_j$, such that $\norm{\widehat{\mu}_c - \mu_c}_{\infty} \leq \nicefrac{\varepsilon}{3}$. A similar guarantee holds for $\AlgDRSK\left(S_j, \frac{\varepsilon}{3}, \tfrac{\delta}{2(\tau+1)}\right)$, which finds $R_j$.

To upper bound the number of samples required by $\AlgDRS$, we note that $\sum_{j=0}^{\tau} n_j \leq 2 n$. This inequality follows from the fact that any arm can be a part of at most two successive slabs. Now, summing over the sample complexities of the subroutines, across all the slabs, we get that overall $O\left(\tfrac{nd}{\varepsilon^2} \log(\tfrac{\tau d}{\varepsilon \delta}) \right)$ samples are drawn in the algorithm. 

Next, we complete the proof by showing that the $\AlgDRS$ achieves the $(\varepsilon, \delta)$-PAC guarantee. Write $\mathcal{E}$ to denote the event that the executions of the subroutines succeed across all the $\tau+1$ slabs. That is, under $\mathcal{E}$, for all the slabs, the subroutines find the desired set of arms along with accurate-enough estimates. Given that the success probability of each subroutine is at least $\left( 1 - \tfrac{\delta}{2(\tau+1)}\right)$ and there are $2(\tau+1)$ subroutine instantiations, we get (via the union bound) that $\mathbb{P} \{ \mathcal{E} \} \geq 1 - \delta$. 

We will prove that, under event $\mathcal{E}$, for every input interval $I_i \in \mathcal{I}$, the set of arms $T_i \subseteq \arm(I_i)$ selected in Line~\ref{line:Pareto-prune} is $\varepsilon$-Pareto optimal for $I_i$; in particular, we establish that  $T_i \subseteq \arm(I_i)$ satisfies both part (a) and (b) of Definition \ref{definition:eps-Pareto-optimal}. Hence, the desired PAC guarantee holds. 

Fix an interval $I_i \in \mathcal{I}$ and consider any arm $b \in \arm(I_i)$. We will first show the that there necessarily exists an arm $a \in T_i$ that satisfies condition (a) (of Definition \ref{definition:eps-Pareto-optimal}) for arm $b$. Towards this, note that Property (P) of our slab construction (see Section~\ref{section:notation}) ensures that interval $I_i$ is partitioned among two or more slabs, $S_x, S_{x+1}, \ldots, S_y$; in particular, $x<y$. Here, $S_x$ is the left-most slab that intersects with $I_i$ and $S_y$ is the rightmost. The intermediate slabs (if any), $S_{x+1}, \ldots ,S_{y-1}$, are completely contained in $I_i$. We will perform a case analysis based on whether arm $b$ is contained in $S_x$ and one of the intermediate slabs, or in $S_y$. 

\noindent
Case {\rm I}: $b \in \arm(S_t)$ with $x \leq t <y$. Here, under event $\mathcal{E}$, the set of arms $L_t$ (obtained by executing $\AlgDSK$ on $S_t$) is an $(\varepsilon/3)$-left-skyline. Therefore, by condition (i) in Proposition \ref{propsition:syline-multi-d-alternative}, there exists an arm $\beta \in L_t$ that satisfies $p_\beta \geq p_b$ and $\mu_\beta \geq  \mu_b - \frac{\varepsilon}{3} \mathbbm{1}$. Note that the right endpoint of $S_t$ is within $I_i$ and, hence, $\beta \in \arm(I_i)$. Furthermore, Line \ref{line:d-candidates} gives us $L_t \subseteq C$. Using these two containments we get $\beta \in C \cap \arm(I_i)$. Note that if $\beta \notin D_i$ (Line \ref{line:Pareto-prune}), then $\beta \in T_i = \left( C \cap \arm(I_i) \right)\setminus D_i$ and setting arm $a = \beta$ we satisfy $\mu_a \geq  \mu_b - \frac{\varepsilon}{3} \mathbbm{1}$, i.e., satisfy part (a) of Definition \ref{definition:eps-Pareto-optimal}. Otherwise, if $\beta \in D_i$, then there exists an arm $a \in \left( C \cap \arm(I_i) \right)\setminus D_i  = T_i$ with the property that $\widehat{\mu}_a > \widehat{\mu}_\beta$--such an arm $a$ is guaranteed to exist. Since, under event $\mathcal{E}$, the estimates for arms $a, \beta \in C$ are $\frac{\varepsilon}{3}$-close to their underlying weights, respectively, we again get that arm $a$ satisfies condition (a): $\mu_a \geq \widehat{\mu}_a - \frac{\varepsilon}{3} \mathbbm{1} > \widehat{\mu}_\beta - \frac{\varepsilon}{3} \mathbbm{1} \geq \mu_\beta - \frac{2\varepsilon}{3} \mathbbm{1} \geq \mu_b - \varepsilon \mathbbm{1}$.

\noindent
Case {\rm II}: $b \in \arm(S_y)$. The analysis here relies on the $(\varepsilon/3)$-right-skyline, $R_y$, of slab $S_y$ and is otherwise identical to the previous case. 
For brevity, we omit repeating the details. 

It only remains to prove that $T_i$ satisfies part (b) of Definition \ref{definition:eps-Pareto-optimal}. Assume, towards a contradiction, that part (b) is violated for an arm $a \in T_i $, i.e., there exists arm $b \in \arm(I_i)$ such that $\mu_b \geq \mu_a + \varepsilon \mathbbm{1}$. Now, as observed above, part (a) is satisfied by $C \cap \arm(I_i)$ with an approximation guarantee of $\varepsilon/3$: there exists arm $\beta \in C \cap \arm(I_i)$ that satisfies $\mu_\beta \geq \mu_b - \frac{\varepsilon}{3} \mathbbm{1}$. Since the estimates for arms $a, \beta \in C$ are $\frac{\varepsilon}{3}$-close to their underlying weights, respectively, we have $\widehat{\mu}_\beta > \widehat{\mu}_a$. However, this strict inequality would imply that $a \in D_i$, i.e., contradict the containment of $a$ in $T_i = \left( C \cap \arm(I_i) \right) \setminus D_i$. 

Therefore, parts (a) and (b) of Definition \ref{definition:eps-Pareto-optimal} hold for $T_i$, proving that $T_i$ is indeed $\varepsilon$-Pareto optimal for $I_i$. This completes the proof. 
\end{proof}

\section{Lower Bound for Bandit Range Searching}
\label{section:lower-bounds}

In this section we will show the sampling upper bounds obtained via our algorithms are essentially tight. Cheu et al.~\cite{cheu2018skyline} established a lower bound for the problem of simultaneously answering many instances of ``find best arm'' (Lemma~\ref{lemma:top-1-T} defines 
it formally).\footnote{Cheu et al.~\cite{cheu2018skyline} prove a more general result, but for our purposes the special case specified in Lemma~\ref{lemma:top-1-T} suffices.} We will present a nontrivial reduction of this problem to bandit range searching with multi-dimensional weights and, hence, establish the desired result. 

\begin{lemma}(\cite{cheu2018skyline})
\label{lemma:top-1-T}
Fix any $T\in \mathbb{Z}_+$.
For any algorithm $\textsc{alg}$, define the following game 
$\game$:
\begin{enumerate}
\item Draw a vector $c=(c_1,\ldots,c_T) \in [m]^T$, where  $c_i$ is 
drawn independently and uniformly at random from $[m]$, for all $1\leq i\leq T$.
\item For every $t\in [T]$, define a set of arms $A_t=\{A_t[1],\ldots,A_t[m] \}$ such that 
$A_t[c_t] \sim {\rm Ber}\left(\frac{1}{4}+2\vare\right)$ and for $i\neq c_t$, $A_t[i] \sim {\rm Ber}\left(\frac{1}{4}\right)$. 
Let ${\cal A}=A_1 \cup \ldots \cup A_T$.
\item $\textsc{Alg}$ outputs a guess $c^{*}$.	
\end{enumerate}
For a sufficiently small constant $C$, if $\textsc{Alg}$ draws fewer than $C\cdot\frac{Tm}{\vare^2}\log\left(\frac{T}{\delta}\right)$ samples in 
expectation, then
\[ \mathbb{P}_{\game} \left\{ c^*=c \right\}< 1-\delta.\]
\end{lemma}

\subsection{Reduction to Bandit Range Searching with Single-Dimensional Weights}
We will first construct a range search instance with set of arms $\mathcal{A}_1$ and interval collection $\mathcal{I}_1$. We will subsequently strengthen the reduction by considering separate copies of $\mathcal{A}_1$ and $\mathcal{I}_1$ in a meta-instance--the subscript will identify the distinct copies.

We start with a game $\gameone$ wherein $T_1=\frac{1}{8\vare}$.\footnote{We assume that $1/8\vare$ is an integer.} 
To construct the bandit instance, we set $\mathcal{A}_1$ to be collection of all arms in the given game, $\mathcal{A}_1 = \cup_{t=1}^{T_1} A_t$. Also,  
we associate with each arm in ${\cal A}_1$ a point on the real line. Specifically, for all $1\leq t\leq T_1$ and all arms $a \in A_t$, we set $p_a = t$. 
To sample an arm in the range searching setup we do the following:   (a) sample
the corresponding arm in the game and  (b) independently draw from ${\rm Ber}(4\vare (t-1))$, and
report the sum of the values  from (a) and (b). 
Indeed, this sampling correspondence and linearity of expectation ensures that (for each $ t \in [T_1]$) the (unknown) weight/mean $\mu_a = \frac{1}{4} + 4\vare(t-1)$ for all arms $a$ in $A_t \subseteq \arm_1$, except for $A_t[c_t]$ whose weight is equal to $\frac{1}{4}+2\vare + 4\vare(t-1)$.\footnote{Recall property 2.~in Lemma~\ref{lemma:top-1-T}, which specifies the underlying distributions associated with the arms in the given game.} It is easy to verify that the weights are in the range $\left[\frac{1}{4},\frac{3}{4}\right] \subseteq [0,1]$.

Write the number of arms $n_1 \coloneqq |\mathcal{A}_1|$. Note that, by construction, $n_1 = m T_1$. In the range searching instance, $\mathcal{I}_1$ is set to be a collection of $T_1$ intervals; in particular, ${\cal I}_1=\{I_1,\ldots, I_{T_1} \}$ with $I_t=[0,t]$. Now we lower bound the number of samples required to solve the bandit range searching problem with arms ${\cal A}_1$ and collection of intervals ${\cal I}_1$.

\begin{lemma}\label{lemma:correctness-1d}
Assume that for the constructed bandit range searching instance---with arms ${\cal A}_1$ and collection of intervals ${\cal I}_1$---the output is correct. Then, the game $\gameone$ can be answered correctly as well. 
\end{lemma}

\begin{proof}
Consider the interval $I_t$ for any $1\leq t\leq T_1$. 
It contains arms of $A_1,\ldots,A_t$ and the only $\vare$-optimal arm among these  is 
$A_t[c_t]$. Hence, a correct output for the bandit range searching problem will 
be $A_t[c_t]$, for all $1\leq t\leq T$. 
Using this output, we will report  $c^{*}=(c_1,c_2,\ldots,c_{T_1})$ as the output of 
the game $\gameone$, which is clearly equal to $c$.
\end{proof}

\begin{theorem}\label{thm:warm-up-lb}
Any algorithm that achieves the $(\varepsilon, \delta)$-PAC guarantee for bandit range searching instances---with $n_1$ arms and single-dimensional weights---has sample complexity $\Omega\left(\frac{n_1}{\vare^2}\log\left(\frac{1}{\vare\delta}\right)\right)$.
\end{theorem}
\begin{proof}
For the sake of contradiction, assume that there exists an algorithm with sample complexity $o\left(\frac{n_1}{\vare^2}\log\left(\frac{1}{\vare\delta}\right)\right)$. 
Then, we can answer the game $\gameone$ as follows: use the reduction 
stated above to construct an instance of  bandit range searching problem  and  
answer the game.

Consider the event where the answer reported by the bandit range searching 
problem is correct. In that case, from Lemma~\ref{lemma:correctness-1d}, it follows that the output of the game 
will also be correct. By our assumption, this implies that the game can be answered 
correctly with probability at least $1-\delta$ with sample complexity 
$o\left(\frac{n_1}{\vare^2}\log\left(\frac{1}{\vare\delta}\right)\right)$; recall that, by construction, $T_1=\frac{1}{8\vare}$ and $n_1 = m T_1$.   
However, this contradicts Lemma~\ref{lemma:top-1-T}. This implies that there are instances of the 
bandit range searching problem for which obtaining a correct solution (with probability  at least
$1-\delta$) requires $\Omega\left(\frac{n_1}{\vare^2}\log\left(\frac{1}{\vare\delta}\right)\right)$ samples.
\end{proof}
 
Next, we strengthen the lower bound to include a dependence on the size of the optimal hitting set. In particular, we fix an integer $\tau \in \mathbb{Z}_+$ and start with a game $\game$ wherein $T=\frac{\tau}{8\vare}$. 

Write parameters $T_1 = T_2 = \ldots T_\tau = \frac{1}{8 \vare}$ and note that $\game$ can be considered as a union of $\tau$ games of the form $\gameone$, since $T_i = T / \tau$ for all  $1 \leq i \leq \tau$. In particular, using the first $T_1 = \frac{1}{8 \vare}$ sets of arms  $A_1 \cup A_2 \cup \ldots \cup A_{1/8\vare}$ in $\game$, we construct the instance $\mathcal{A}_1$ and $\mathcal{I}_1$ as mentioned above. 

Next, we construct instances with ${\cal A}_i$ and ${\cal I}_i$, for all $2\leq i\leq \tau$. ${\cal I}_i$ is obtained by translating (performing a right shift) each interval in ${\cal I}_{i-1}$ by $1+\frac{1}{8\vare}$ on the real line. Set ${\cal A}_i$ will encode 
arms in $A_{\frac{i-1}{8\vare} +1} \cup A_{\frac{i-1}{8\vare}+2} \cup \ldots \cup A_{\frac{i-1}{8\vare} + 1/8\vare}$.
The $p$-value ($p_a$ for arm $a$) of arms in $A_{\frac{i-1}{8\vare}+t}$ is equal to the $p$-value of 
the arms in $A_{\frac{i-2}{8\vare}+t}$ plus the quantity $1+\frac{1}{8\vare}$. The weight of the arms 
in $A_{\frac{i-1}{8\vare}+t}$ is equal to $\frac{1}{4} +4\vare(t-1)$, except for 
$A_{\frac{i-1}{8\vare}+t}\left[c_{\frac{i-1}{8\vare}+t}\right]$ 
whose weight is equal to $\frac{1}{4}+2\vare + 4\vare(t-1)$. 
Note that $|\mathcal{A}_i| = m T_i = \frac{m}{8 \vare}$ and $|\mathcal{A}| = \sum_{i=1}^\tau |\mathcal{A}_i|  = \frac{m \tau}{8 \vare} = mT$.

We now establish sample lower bounds for the bandit range searching instance with arms ${\cal A} = \bigcup_{ i=1}^\tau {\cal A}_i$ and intervals $\mathcal{I}= \bigcup_{ i=1}^\tau {\cal I}_i$.

\begin{lemma}\label{lemma:hitting-set}
The following properties hold true:
\begin{enumerate}
\item ${\cal I}$ admits an optimal hitting set of size $\tau$.	
\item The arms lying inside $j$-th interval of ${\cal I}_i$ are 
$A_{\frac{i-1}{8\vare} +1}, A_{\frac{i-1}{8\vare}+2} \ldots ,  A_{\frac{i-1}{8\vare} + j}$, 
for all $1\leq i\leq \tau$ and $1\leq j\leq \frac{1}{8\vare}$. 	
\end{enumerate}
\end{lemma}

\begin{proof}
To prove property~(1), observe that for any $i\neq j$, any interval in ${\cal I}_{i}$ is disjoint from any interval in ${\cal I}_{j}$. As a result, at least $\tau$ 
points are needed to hit all the intervals in ${\cal I}$.  On the other hand, 
for any $1\leq i \leq \tau$, all the intervals 
in  $I_{i}$ share the same left endpoint, and as such, these
$\tau$ left endpoints will form an optimal hitting set of ${\cal I}$.  

Property~(2) is established via induction. The property is true for 
$({\cal A}_1,{\cal I}_1)$, which is the base case. By induction hypothesis, 
assume that the property holds for ${\cal I}_{i-1}$. Then the $j$-th interval 
in ${\cal I}_{i-1}$ contains arms 
$A_{\frac{i-2}{8\vare} +1}, A_{\frac{i-2}{8\vare}+2} \ldots ,  A_{\frac{i-2}{8\vare} + j}$.
Since the $j$-th interval  in ${\cal I}_i$ is obtained by translating the 
$j$-th interval in ${\cal I}_{i-1}$ by $1+\frac{1}{8\vare}$, and the coordinate value 
of $A_{\frac{i-1}{8\vare} + j}$ is exactly $1+\frac{1}{8\vare}$ larger than 
$A_{\frac{i-2}{8\vare} + j}$, the property holds for ${\cal I}_i$ as well. 
\end{proof}

\begin{lemma}\label{lemma:correctness-1d-full}
Assume that for the constructed bandit range searching instance---with arms ${\cal A}$ and collection of intervals ${\cal I}$---the output is correct. Then, the game $\game$ can be answered correctly as well.  
\end{lemma}
\begin{proof}
Consider the interval $j$-th interval in ${\cal I}_i$ for any $1\leq i \leq \tau$ and any $1\leq j\leq 1/8\vare$. 
By property~(2) of Lemma~\ref{lemma:hitting-set}, it contains arms of $A_{\frac{i-1}{8\vare} +1}, A_{\frac{i-1}{8\vare}+2} \ldots ,  A_{\frac{i-1}{8\vare} + j}$ and the only $\vare$-optimal arm among these  is 
$A_{\frac{i-1}{8\vare}+j}[c_{\frac{i-1}{8\vare}+j}]$. 
Hence, a correct output for the bandit range searching problem will 
be $A_t[c_t]$, for all $1\leq t\leq T$. 
Using this output, we will report  $c^{*}=(c_1,c_2,\ldots,c_T)$ as the output of 
the game $\game$ which is clearly equal to $c$.
\end{proof}

\begin{theorem}
For each $\tau \in \mathbb{Z}_+$, there exists a bandit range searching instance with $n$ arms and 
 collection of intervals $\mathcal{I}$ such that the size of optimal hitting set for $\mathcal{I}$ is $\tau$, and any algorithm that achieves the $(\varepsilon, \delta)$-PAC guarantee for the instance necessarily draws $\Omega\left(\frac{n}{\vare^2}\log\left(\frac{\tau}{\vare \delta}\right)\right)$ samples. 
\end{theorem}
\begin{proof}
The proof is analogous to the proof of Theorem~\ref{thm:warm-up-lb} with $T=\tau/8\vare$ and $n = mT$. Note that $\tau$ is indeed the size of the optimal hitting set for $\mathcal{I}$ (property~(1) of Lemma~\ref{lemma:hitting-set}).
\end{proof}

\subsection{Reduction to Bandit Range Searching with Multi-Dimensional Weights}
 At a high-level, our strategy will be similar to the single-dimensional setting. 
 We will first construct a range search instance with set of arms $\mathcal{A}_1$ 
 and interval collection $\mathcal{I}_1$, and  subsequently strengthen the reduction by 
 considering  separate copies of $\mathcal{A}_1$ and $\mathcal{I}_1$ in a meta-instance.
  
We fix an ambient dimension $d$ for the weights and start with a game $\gameone$ wherein $T_1 = \left(\frac{1}{8\vare}\right)^d$. 
We construct a bandit range setting instance (with $d$-dimensional weights) by setting $\mathcal{A}_1$ to be the set of all the arms in the game. 
We associate a point with each arm in $\arm_1$ as follows: for all $1\leq t\leq T_1$ and each arm $a \in A_t$,
 set $p_a = t$. 
 
To sample a particular arm  from $A_t$ in the range searching setup we do the following:
let $k_1 \  k_2\ldots k_d$ be the representation of integer $(t-1)$ in base-$\left(\frac{1}{8\vare}\right)$ 
(for example, base-$2$ implies binary representation). Then,
(a) sample the corresponding arm in the game
and (b) independently draw from the following multivariate Bernoulli distribution $({\rm Ber}(4\vare k_1),\ldots, {\rm Ber}(4\vare k_d))$. 
Let $\alpha$ and $(\beta_1,\ldots,\beta_d)$ be the output from (a) and (b), respectively.
Then, sampling that arm in $A_t \subseteq \arm_1$  will report $(\alpha+\beta_1,\ldots, \alpha + \beta_d)$.
This sampling correspondence and linearity of expectation ensure that (for each $ t \in [T_1]$) the (unknown) weight/mean $\mu_a = (\frac{1}{4} + 4\vare k_1,\ldots, 
\frac{1}{4} + 4\vare k_d)$ for all arms $a$ in $A_t$, except for $A_t[c_t]$ whose weight is equal to 
$(\frac{1}{4}+2\vare + 4\vare k_1, \ldots, \frac{1}{4}+2\vare + 4\vare k_d)$. 
It is easy to verify that the weights are in the range $\left[\frac{1}{4},\frac{3}{4}\right]^d \subseteq [0,1]^d$.

Write the number of arms $n_1 \coloneqq |\mathcal{A}_1|$. Note that, by construction, $n_1 = m T_1$. 
The query intervals in the bandit 
range searching instance will be the set of intervals ${\cal I}_1=\{I_1,\ldots, I_{1/(8\vare)^d}\}$ 
 such that $I_t=[0,t]$. Now we lower bound the number of samples required to solve the bandit range searching problem with arms ${\cal A}_1$ and collection of intervals ${\cal I}_1$.

\begin{lemma}\label{lemma:single-arm}
For every $1\leq t\leq T_1$, the $\vare$-Pareto optimal set of $A_t$ is $\{A_t[c_t]\}$.
\end{lemma}
\begin{proof}
For each $t$, the weight of $A_t[c_t]$ is componentwise greater than the weight $A_t[i]$ by $2\vare$, for $i \neq c_t$.
The first condition of $\vare$-Pareto optimality (Definition~\ref{definition:eps-Pareto-optimal}) ensures that 
$A_t[c_t]$ is part of the $\vare$-Pareto optimal set. Also, focussing on the second condition of $\vare$-Pareto optimality with respect to $A_t[c_t]$,  we get that none of the other arms in $A_t$ can be part of the $\vare$-Pareto optimal set.
\end{proof}

Before going further, we will define {\em $\vare$-domination}: an arm $a$ is said to $\vare$-dominate another arm $b$ iff $\mu_a \geq \mu_b -\vare \mathbbm{1}$. Recall that $\mu_x$ is the $d$-dimensional weight of an arm $x$ and $\mathbbm{1}$ denotes the all-ones vector.

\begin{lemma}\label{lemma:cannot-dominate}
Given two integers $t'$ and $t$ such that $1\leq t'<t\leq T_1$,  arm $A_{t'}[c_{t'}]$ cannot $\vare$-dominate arm $A_{t}[c_{t}]$.
This also implies that none of the arms in $A_{t'}$ can $\vare$-dominate arm $A_t[c_t]$. 
\end{lemma}
\begin{proof}
 Let $k_1k_2\ldots k_d$ (respectively, $k_1'k_2'\ldots k_d'$) 
be the representation of $(t-1)$ (respectively, $(t'-1)$) in base-$\left(\frac{1}{8\vare}\right)$.
Since $t'< t$, there exists an index $\ell \in [d]$ such that such that 
\begin{align*}
k_{\ell}' < k_{\ell} & \implies k_{\ell}' \leq k_{\ell}-1 \implies 4\vare k_{\ell}' \leq 4\vare k_{\ell} -4\vare \\
& \implies 4\vare k_{\ell}' + \vare < 4\vare k_{\ell} .\\
\end{align*}

This inequality for the $\ell$-th coordinate ensures that the arm $A_{t'}[c_{t'}]$ cannot $\vare$-dominate arm $A_{t}[c_{t}]$.
 \end{proof}

\begin{lemma}\label{lemma:correctness-simple-d}
Assume that for the constructed bandit range searching instance---with arms ${\cal A}_1$ and collection of intervals ${\cal I}_1$---the output is correct. Then, the game $\gameone$ can be answered correctly as well.
\end{lemma}

\begin{proof}
For any $1\leq t\leq T_1$, the interval $I_t$ 
will contain arms in $A_1,\ldots,A_t$. We claim that the $\vare$-Pareto optimal 
set of $A_1\cup A_2 \cup \ldots \cup A_t$ includes $A_t[c_t]$: firstly, by Lemma~\ref{lemma:single-arm}, 
the $\vare$-Pareto optimal set of $A_{t}$ will be $\{A_t[c_t]\}$ and, hence, all the other 
arms in $A_t$ can be discarded. Next, by Lemma~\ref{lemma:cannot-dominate}, 
no arm from $A_1 \cup \ldots \cup A_{t-1}$ 
can $\vare$-dominate $A_t[c_t]$.  This implies that 
 $A_t[c_t]$ will necessarily be part of the output for $I_t$; in particular, the intersection of the output for $I_t$ and $A_t$ will be $\{ A_t[c_t] \}$. 

Therefore, using the output for the bandit range searching instance, we can report $c^*=(c_1,\ldots,c_T)$ as the answer for the game $\gameone$,  which is equal to $c$ and, hence, answers the game correctly. 
\end{proof}

\begin{theorem}
Any algorithm that achieves the $(\varepsilon, \delta)$-PAC guarantee for bandit range searching instances---with $n_1$ arms and multi-dimensional weights---has sample complexity 
$\Omega\left(\frac{n_1}{\vare^2}\log\left(\frac{1}{\vare^d\delta}\right)\right)$.
\end{theorem}
\begin{proof}
The proof is analogous to the proof of Theorem~\ref{thm:warm-up-lb} with $T_1=\left(\frac{1}{8\vare}\right)^d$.
\end{proof}

In the one-dimensional case, we strengthened the lower bound 
from $\Omega\left(\frac{n_1}{\vare^2}\log\left(\frac{1}{\vare \delta}\right)\right)$ to 
$\Omega\left(\frac{n}{\vare^2}\log\left(\frac{\tau}{\vare\delta}\right)\right)$ by 
setting $T=\frac{\tau}{8\vare}, n=mT$ and by translating the intervals and the arms 
by $1+ \frac{1}{8\vare}$. For the multi-dimensional case, we will follow exactly 
the same approach to strengthen the lower bound 
from $\Omega\left(\frac{n_1}{\vare^2}\log\left(\frac{1}{\vare^d \delta}\right)\right)$ to 
$\Omega\left(\frac{n}{\vare^2}\log\left(\frac{\tau}{\vare^d\delta}\right)\right)$ by
setting the parameters $T=\frac{\tau}{(8\vare)^d}, m=n/T$ and by translating the intervals and the arms 
by $1+ \frac{1}{(8\vare)^d}$. This establishes Theorem~\ref{thm:final-lb}.

\begin{restatable}{theorem}{TheoremLowerBound}
\label{thm:final-lb}
For each $\tau \in \mathbb{Z}_+$, there exists a bandit range searching instance, with $n$ arms and interval collection $\mathcal{I}$, such that the size of optimal hitting set for $\mathcal{I}$ is $\tau$, and any algorithm that achieves the $(\varepsilon, \delta)$-PAC guarantee for the instance necessarily draws $\Omega\left(\frac{n}{\vare^2}\log\left(\frac{\tau}{\vare^d\delta}\right)\right)$ samples. Here, $d$ is the ambient dimension of the weights. 
\end{restatable}

\section{Conclusion and Future Work}
This work establishes tight sample-complexity bounds for the problem of identifying optimal arms within a given collection of intervals. Complementing the current focus on sample-efficient algorithms, one can also consider a \emph{regret} version of the range searching. Specifically,  developing algorithms that obtain sublinear regret across all input intervals is an interesting direction of future work. 

Extending the range aspects to higher dimensions (e.g., finding optimal arms within rectangles) would also be interesting. More generally, the interplay of computational geometry and multi-armed bandits stands as a rich source of geometric problems over uncertain data.

\section*{Acknowledgements}
Siddharth Barman gratefully acknowledges the support of a Ramanujan Fellowship (SERB - {SB/S2/RJN-128/2015}). 
Saladi Rahul's research is generously supported by an IISc Startup Grant.

\bibliographystyle{alpha}
\bibliography{range-bandits}

\appendix
\section{Missing Proofs from Section \ref{section:multi-dimension-algos}}
\label{appendix:missing-proofs-mutli-dimension-algos}

\subsection{Proof of Proposition \ref{propsition:syline-multi-d-alternative}}
\label{appendix:skyline-equivalence-proof}

\PropositionSkylineCharacter*

\begin{proof}
First, we will show that condition (i) in Definition~\ref{definition:skyline-multi-d} and the proposition, respectively, are equivalent. 

\noindent
\emph{Condition (i):} For the forward direction (i.e., from Definition \ref{definition:skyline-multi-d} to Proposition \ref{propsition:syline-multi-d-alternative}),
consider any arm $b \in \arm(J)$. Since $L$ contains $\varepsilon$-Pareto optimal arms $T$ for the interval $[p_{b}, r]$, we get (invoking Definition \ref{definition:eps-Pareto-optimal}) that there necessarily exists an arm $\beta \in T \subseteq L$ that satisfies $\mu_\beta  \geq  \mu_b - \varepsilon \mathbbm{1}$. Since all the arms in $T$ (in particular, $\beta$) lie within $[p_b, r]$, we also have $p_\beta \geq p_b$. 

\noindent
For the reverse direction, consider any arm $b \in \arm(J)$ and let $\arm([p_b, r])$ denote all the arms in the interval $[p_b, r]$, i.e., $\arm([p_b, r])  = \{x \in \arm(J) \mid p_x \geq p_b\}$. Condition (i) of the proposition (applied to each arm $b' \in \arm([p_b, r])$) implies that for every arm $b' \in \arm([p_b, r])$ there exists an arm $\beta' \in L$ such that $p_{\beta'} \geq p_{b'}$ and $\mu_{\beta'} \geq \mu_{b'} - \varepsilon \mathbbm{1}$. The set of all such arms $\beta'$s encapsulate an $\varepsilon$-Pareto optimal subset $T$ for the interval $[p_b, r]$. \\

To complete the proof we will establish the equivalence of condition (ii) in Definition~\ref{definition:skyline-multi-d} and the proposition, respectively. 

\noindent
\emph{Condition (ii):}
Going from the definition to the proposition, consider any arm $\beta \in L$ and let $b \in \arm(J)$ be an arm such that $\beta$ is $\varepsilon$-Pareto optimal in $[p_b, r]$. Hence, by Definition \ref{definition:eps-Pareto-optimal}, we get that there exists no arm $y \in \arm([p_b, r])$ with the property that $\mu_y \geq \mu_\beta + \varepsilon \mathbbm{1}$. Since $p_\beta \geq p_b$, we have $\arm([p_b, r]) \supseteq \arm([p_\beta, r]) = \{ x \in \arm(J) \mid p_x \geq p_\beta \}$. Therefore, in particular, there does not exist an arm $x \in \arm(J)$ such that $\mu_x \geq \mu_\beta + \varepsilon \mathbbm{1}$ and $p_x \geq p_\beta$. \\
Finally, note that if for an arm $\beta$ there does not exist an arm $x \in \arm(J)$ such that $p_x \geq p_\beta$ and $\mu_x \geq \mu_\beta + \varepsilon \mathbbm{1}$, then $\beta$ is $\varepsilon$-Pareto optimal in $[p_\beta, r]$ itself. That is, condition (ii) of the proposition implies condition (ii) in Definition \ref{definition:skyline-multi-d}. 
\end{proof}

\subsection{Proof of Lemma \ref{lemma:iterative-pac-guarantee}}
\label{appendix:reps}

\LemmaRepsToReps*
\begin{proof}
We will prove the lemma in two steps. 

\noindent
({\rm I}): First, we will establish that (with high probability) for every representative arm $r \in R_t$, there exists an arm $r' \in A_{t+1}$ with the properties that $p_{r'} \geq p_r$ and $\mu_{r'}   \geq \mu_r - \tfrac{3 \varepsilon_t}{4}\mathbbm{1}$. 

\noindent 
({\rm II}): Then, we show that for such an $r' \in A_{t+1}$, there exists an arm $\gamma \in R_{t+1}$ that satisfies $p_{\gamma} \geq p_{r'}$ and $\mu_{\gamma} \geq \mu_{r'}  - \tfrac{\varepsilon_t}{4} \mathbbm{1}$. 

The two steps in conjunction ensure the existence of the desired arm $\gamma \in R_{t+1}$ and, hence, will complete the proof. 

We will prove that ({\rm I}) and ({\rm II}) will hold under the following high-probability events: write $\mathcal{E}_1$ to denote the event that $\norm{\mu_x - \muhat_x}_{\infty} \leq \tfrac{\varepsilon_t}{4}$ for all arms $x \in R_t$, and let $\mathcal{E}_2$ denote the event that for at most $\tfrac{|A_t|}{30} \left(\tfrac{\varepsilon_t}{4}\right)^d$ arms $f \in A_t$ we have $\norm{\mu_f - \muhat_f}_{\infty} \geq \tfrac{\varepsilon_t}{4}$, i.e., $\mathcal{E}_2$ denotes $\left| \left\{ f \in A_t \ : \ \norm{\mu_f - \muhat_f}_{\infty} \geq \tfrac{\varepsilon_t}{4} \right\}\right| \leq \tfrac{|A_t|}{30} \left(\tfrac{\varepsilon_t}{4}\right)^d$

\begin{claim}
\label{claim:bad-event-unlikely}
$\mathbb{P}\left\{\mathcal{E}_1 \cap \mathcal{E}_2 \right\} \geq 1 - \delta_t$
\end{claim}
\begin{proof}
Each arm $a \in A_t$ is pulled $\tfrac{8}{\varepsilon_t^2} \log \left( \left(\tfrac{4}{\varepsilon_t}\right)^d \tfrac{50d}{\delta_t} \right)$ times in Line \ref{line:pull} of the algorithm. Hence, Hoeffding's inequality and union bound (across the $d$-dimensions of $\mu_a$) give us, for each arm $a \in A_t$: 
\begin{align}
\mathbb{P}\left\{ \norm{\muhat_a - \mu_a}_{\infty} \geq  \dfrac{\varepsilon_t}{4} \right\} \leq \left(\dfrac{\varepsilon_t}{4}\right)^d \dfrac{\delta_t}{50} \label{inequality:hoeffding-result}
\end{align}
 
Recall that the size of $R_t \subseteq A_t$ is at most $\left(\tfrac{4}{\varepsilon_t}\right)^d$. Hence, applying union bound, across all the arms in $R_t$, we get that $\mathbb{P}\left\{ \mathcal{E}_1 \right\} \geq 1-\tfrac{\delta_t}{50}$. 

Let the random variable $F$ denote the number of arms $f \in A_t$ for which $\norm{\muhat_f - \mu_f}_{\infty} \geq \tfrac{\varepsilon_t}{4}$. Event $\mathcal{E}_2$ corresponds to $F \leq \tfrac{|A_t|}{30} \left(\tfrac{\varepsilon_t}{4}\right)^d$. Inequality (\ref{inequality:hoeffding-result}) implies that the expected value of $F$ is at most $|A_t| \left(\tfrac{\varepsilon_t}{4}\right)^d \tfrac{\delta_t}{50}$. 
Now, via Markov's inequality, we get $\mathbb{P} \left\{ F > \tfrac{|A_t|}{30} \left(\tfrac{\varepsilon_t}{4}\right)^d \right\} \leq \tfrac{|A_t| \left(\nicefrac{\varepsilon_t}{4}\right)^d \nicefrac{ \delta_t}{50}}{\nicefrac{|A_t|}{30} \left(\nicefrac{\varepsilon_t}{4}\right)^d} = \frac{30\delta_t}{50}$. That is, $\mathbb{P} \left\{ \mathcal{E}_2 \right\} \geq 1-\tfrac{30\delta_t}{50}$. Therefore, $\mathbb{P}\left\{\mathcal{E}_1 \cap \mathcal{E}_2 \right\} \geq 1 - \delta_t$, and the claim follows. 
\end{proof}
For the remainder of the proof we will assume that events $\mathcal{E}_1$ and $\mathcal{E}_2$ hold, and show that they imply the above-mentioned conditions ({\rm I}) and ({\rm II}). 

For ({\rm I}), consider any arm $r \in R_t \subseteq A_t$ and write $\widehat{B}$ to denote the cube that contains $\widehat{\mu}_r$, i.e., $r \in \widehat{B}(A_t)$ in Line \ref{line:bat}. We perform a case analysis based on the size of $\widehat{B}(A_t)$.

\noindent
Case 1: $|\widehat{B}(A_t)| \leq \frac{|A_t|}{10} \left(\dfrac{\varepsilon_t}{4}\right)^d$. In this case $r$ is included in $A_{t+1}$ (see the cardinality threshold in Line~\ref{line:dropping-threshold}) and, hence, setting $r' = r \in A_{t+1}$ satisfies ({\rm I}). 

\noindent
Case 2: $|\widehat{B}(A_t)| > \frac{|A_t|}{10} \left(\dfrac{\varepsilon_t}{4}\right)^d$. Here,  if $p_r$ is greater than the median of $\{p_y\}_{y \in \widehat{B}(A_t)}$, then $r$ is included in the set $A_{t+1}$; see Lines \ref{line:dropping-threshold} and \ref{line:elimination}. Hence, as before, setting $r'=r \in A_{t+1}$ suffices. The remaining sub-case is when $p_r$ is less than the median. Note that, even in such a situation, at least $\frac{|A_t|}{20} \left(\dfrac{\varepsilon_t}{4}\right)^d$ arms (in particular, half of the arms in $\widehat{B}(A_t)$, whose $p_y$ values are above the median) are included in $A_{t+1}$. That is, $|\widehat{B}(A_t) \cap A_{t+1}| \geq \frac{|A_t|}{20} \left(\dfrac{\varepsilon_t}{4}\right)^d$. Event $\mathcal{E}_2$ ensures that for at least $L \coloneqq \left( \frac{|A_t|}{20} \left(\frac{\varepsilon_t}{4}\right)^d - \frac{|A_t|}{30} \left(\frac{\varepsilon_t}{4}\right)^d\right)$ arms $y \in \widehat{B}(A_t) \cap A_{t+1}$ we have $\| \mu_y - \widehat{\mu}_y \|_\infty \leq \frac{\varepsilon_t}{4}$.\footnote{Note that, via the loop condition in Line \ref{line:while-loop-condition}, we have $L \geq 1$.} Select anyone of these $L$ arms as $r' \in \widehat{B}(A_t) \cap A_{t+1}$. We will show that such an $r'$ satisfies ({\rm I}) and, hence, complete the case analysis. Since arm $r \in \widehat{B}(A_t) $ was dropped and $r' \in \widehat{B}(A_t) $ was not, we have $p_{r'} \geq p_r$. Next, the triangle inequality gives us
\begin{align*}
\norm{ \mu_{r} - \mu_{r'} }_{\infty} \leq & \norm{ \mu_{r} - \muhat_{r} }_{\infty} + \norm{ \muhat_{r} - \muhat_{r'}}_{\infty} + \norm{ \muhat_{r'} - \mu_{r'} }_{\infty} \\
\leq & \tfrac{\varepsilon_t}{4} + \norm{ \muhat_{r} - \muhat_{r'}}_{\infty} + \norm{ \muhat_{r'} - \mu_{r'} }_{\infty} \tag{$r\in R_t$ and $\mathcal{E}_1$ holds} \\
\leq & \tfrac{\varepsilon_t}{4} + \tfrac{\varepsilon_t}{4} + \norm{ \muhat_{r'} - \mu_{r'} }_{\infty} \tag{$\muhat_{r}, \muhat_{r'} \in \widehat{B}$} \\
\leq & \tfrac{\varepsilon_t}{4} + \tfrac{\varepsilon_t}{4} + \tfrac{\varepsilon_t}{4} = \tfrac{3 \varepsilon_t}{4} \tag{choice of $r'$}
\end{align*}
Therefore, in this final (sub) case as well we have an arm $r' \in A_{t+1}$ that satisfies ({\rm I}).

To establish ({\rm II}), let $B' \in \mathcal{B}_{t+1}$ denote the cube that contains $\mu_{r'}$.\footnote{Note that $B'$ is defined based on the exact weight $\mu_{r'}$ and not an estimate.}  The representative $r(B') \in R_{t+1}$ of cube $B'$ is the desired arm $\gamma$,  since with $\gamma = r(B') \in R_{t+1}$ we obtain $p_\gamma \geq p_{r'}$ and $\mu_{\gamma}  \geq \mu_{r'} - \tfrac{\varepsilon_{t+1}}{4} \mathbbm{1} \geq \mu_{r'} - \tfrac{\varepsilon_{t}}{4} \mathbbm{1}$.

Overall, we have shown that events $\mathcal{E}_1$ and $\mathcal{E}_2$ hold with probability $1 - \delta_t$, and these events imply ({\rm I}) and ({\rm II}). Therefore, with probability $1 - \delta_t$, there exists an arm $\gamma \in R_{t+1}$ that satisfies the requirement in the lemma statement. This completes the proof. 
\end{proof}

\subsection{Proof of Lemma \ref{lemma:skyline-sample-complexity}}
\label{appendix:skyline-sample-complexity-proof}

\LemmaSkylineSampleComplexity*
\begin{proof}
First, we show that the number of arms under consideration drops geometrically in each iteration of $\AlgDSK$. 

\begin{claim}
\label{claim:arms-geometric-shrink-arms}
For any two consecutive iterations, $t$ and $t+1$, of $\AlgDSK$, we have $|A_{t+1}| \leq \tfrac{11}{20}|A_t|$.
\end{claim}

\begin{proof}
Write $\widetilde{\mathcal{B}}_t$ to denote the collection of cubes that do \emph{not} satisfy the if-condition in Line \ref{line:dropping-threshold}, i.e., $\widetilde{\mathcal{B}}_t \coloneqq \left\{ \widetilde{B} \in \mathcal{B}_t \ : \ |\widetilde{B}(A_t)| \leq \tfrac{|A_t|}{10} \left(\tfrac{\varepsilon_t}{4}\right)^d \right\}$. Since $|\widetilde{\mathcal{B}}_t| \leq |{\mathcal{B}}_t| \leq \left(\nicefrac{4}{\varepsilon_t}\right)^d$ and the number of arms associated with each cube in $\widetilde{\mathcal{B}}_t$ is at most $\tfrac{|A_t|}{10} \left(\tfrac{\varepsilon_t}{4}\right)^d$, we get that number of arms that are \emph{not} considered in Line \ref{line:dropping-threshold} is at most $\tfrac{|A_t|}{10} \left(\tfrac{\varepsilon_t}{4}\right)^d \times \left(\frac{4}{\varepsilon_t}\right)^d = \tfrac{|A_t|}{10}$. 

The remaining arms are at least $\frac{9|A_t|}{10}$ in number, and half of them are dropped (by considering the medians, cube wise). That is, at least $\frac{9|A_t|}{20}$ arms are dropped ($|D_t| \geq \frac{9|A_t|}{20}$) and we get the desired bound $|A_{t+1}| \leq \tfrac{11}{20}|A_t|$. 
\end{proof}

Write $T$ to denote the total number of iterations of the while loop in $\AlgDSK$. Since $A_1$ is initialized as the set of arms within the given interval $J$, we have $|A_1|=m=|\arm(J)|$. A repeated application of Claim \ref{claim:arms-geometric-shrink-arms} gives us $|A_t| \leq \left(\frac{11}{20}\right)^{t-1} m$, for all $1 \leq t \leq T$. In addition, the geometric update of the approximation and the confidence parameters ensure $\varepsilon_t = \left(\nicefrac{3}{4}\right)^{t-1} \nicefrac{\vare}{5}$ and $\delta_t = \left(\nicefrac{1}{2}\right)^{t-1} \nicefrac{\delta}{2}$. We also note that the exit condition of the while loop ensures $|A_{T}| \leq 120\left(\frac{4}{\varepsilon_T}\right)^d$.
Using these observations, we upper bound the number of samples drawn by the algorithm:

\begin{align*}
    & \underbrace{\sum_{t=1}^{T-1} \dfrac{8 |A_t|}{\varepsilon_t^2} \log \left( \left(\dfrac{4}{\varepsilon_t}\right)^d \dfrac{50d}{\delta_t} \right)}_{\text{While loop}} + \underbrace{\dfrac{|A_{T}|}{\varepsilon_{T}^2} \log \left(\dfrac{|A_{T}|d}{\delta_{T}} \right)}_{\text{Line \ref{line:pruning}}} \\
    \leq & \sum_{t=1}^{T-1} \dfrac{8|A_t|}{\varepsilon_t^2} \log \left( \left(\dfrac{4}{\varepsilon_t}\right)^d \dfrac{50d}{\delta_t} \right) + \dfrac{|A_{T}|}{\varepsilon_T^2} \log \left(\left(\dfrac{4}{\varepsilon_T}\right)^d\dfrac{120d}{\delta_T} \right) \\
    \leq & \sum_{t=1}^{T} \dfrac{8|A_t|}{\varepsilon_t^2} \log \left( \left(\dfrac{4}{\varepsilon_t}\right)^d \dfrac{120d}{\delta_t} \right) \\
    \leq & 8 \sum_{t=1}^{T} \dfrac{m \left(\nicefrac{11}{20}\right)^{t-1}}{\left((\nicefrac{3}{4})^{t-1}\nicefrac{\varepsilon}{5}\right)^2} \log \left( \left(\dfrac{4}{(\nicefrac{3}{4})^{t-1}\nicefrac{\varepsilon}{5}}\right)^d \dfrac{120d}{ \left(\nicefrac{1}{2}\right)^{t-1} \nicefrac{\delta}{2} } \right) \\
     =& \dfrac{200m}{\varepsilon^2} \sum_{t=1}^{T} \left(\frac{44}{45}\right)^{t-1} \log \left( \dfrac{240 d }{\delta} \left(\frac{20}{\varepsilon}\right)^d  \left(\frac{4}{3}\right)^{d(t-1)} 2^{t-1}\right) \\
     \leq & \dfrac{200md}{\varepsilon^2}  \sum_{t=1}^{T} \left(\frac{44}{45}\right)^{t-1} \log \left( \frac{4800d}{\varepsilon \delta} \left(\frac{8}{3}\right)^{t-1}\right) \\
     =& \dfrac{200md}{\varepsilon^2}  \sum_{t=1}^{T} \left(\frac{44}{45}\right)^{t-1} \Biggl( \log(4800) + (t-1) \log(\nicefrac{8}{3})\\ 
     & \quad \quad \quad \quad \quad \quad \quad \quad \quad \quad + \log\left(\frac{d}{\varepsilon \delta} \right) \Biggr) \\
     \leq & \dfrac{200md}{\varepsilon^2} \log\left(\frac{d}{\varepsilon \delta} \right)  \sum_{t=1}^{\infty} \left(\frac{44}{45}\right)^{t-1} \left( C + (t-1) C'\right) \\
     = & O\left(\frac{md}{\varepsilon^2} \log \left(\frac{d}{\varepsilon \delta} \right) \right) \\
\end{align*}

$C$ and $C'$ in the penultimate step are constants, and it is assumed $\nicefrac{d}{\varepsilon \delta} \geq e$; without this condition the final $\log \left(\nicefrac{d}{\varepsilon \delta}\right)$ term can be trivially dropped. 

Overall, we get that $\AlgDSK$ draws $O\left(\frac{md}{\varepsilon^2}\log \frac{d}{\varepsilon \delta}\right)$ samples and the lemma follows.  
\end{proof}

\end{document}